\documentclass{article} % For LaTeX2e
\usepackage{iclr2024_conference}
\usepackage{times}
\usepackage{latexsym}
\usepackage{inconsolata}
\usepackage[utf8]{inputenc} % allow utf-8 input
\usepackage[T1]{fontenc}    % use 8-bit T1 fonts
\usepackage{hyperref}       % hyperlinks
\usepackage{url}            % simple URL typesetting
\usepackage{booktabs}       % professional-quality tables
\usepackage{amsfonts}       % blackboard math symbols
\usepackage{nicefrac}       % compact symbols for 1/2, etc.
\usepackage{microtype}      % microtypography
\usepackage{lipsum}         % Can be removed after putting your text content
\usepackage{graphicx}
\usepackage{natbib}
\usepackage{doi}
\usepackage{booktabs} % for professional tables
\usepackage{amsmath}
\usepackage{amssymb}
\usepackage{mathtools}
\usepackage{amsthm}
\usepackage{subfigure}
\usepackage{fancyhdr}
\usepackage{svg}
\usepackage{xcolor} % changed from color to xcolor (2021/11/24)
\usepackage{algorithmic}
\usepackage{algorithm}
\usepackage{amsfonts}       % blackboard math symbols
\usepackage{nicefrac}       % compact symbols for 1/2, etc.
\usepackage{microtype}      % microtypography
\usepackage{hyperref}
\usepackage{tablefootnote}
\usepackage{threeparttable}
\usepackage{wrapfig}
\usepackage[framemethod=TikZ]{mdframed}
\usepackage{enumitem}
\usepackage{adjustbox}

\usepackage[capitalize,noabbrev]{cleveref}

\usepackage{eso-pic} % used by \AddToShipoutPicture
\usepackage{url}
\usepackage{multirow}
\newif\iffinal
\iffinal
    \newcommand{\tianhao}[1]{}
    \newcommand{\tong}[1]{}
    \newcommand{\ashwinee}[1]{}
    \newcommand{\prateek}[1]{}
    \newcommand{\addressedprateek}[1]{}
    \newcommand{\add}[1]{}
\else
    \newcommand{\tianhao}[1]{{\bf \textcolor{purple}{[Tianhao: #1]}}}
    \newcommand{\tong}[1]{{\bf \textcolor{purple}{[Tong:#1]}}}
    \newcommand{\ashwinee}[1]{{\bf \textcolor{purple}{[Ashwinee:#1]}}}
    \newcommand{\prateek}[1]{{\color{red} \textbf{Prateek: #1}}}
    \newcommand{\addressedprateek}[1]{{\color{blue} \textbf{(Addressed) Prateek: #1}}}
    \newcommand{\add}[1]{\textcolor{red}{#1}}
\fi

\usepackage{amsmath,amsfonts,bm}

\newtheorem{theorem}{Theorem}
\newtheorem{lemma}[theorem]{Lemma}

\newtheorem{definition}[theorem]{Definition}

\newtheorem{remark-star}{Remark}
\newtheorem{remark-star-1}{Remark}
\usepackage{booktabs}

\DeclareMathOperator*{\argmax}{\arg\!\max}

\newcommand{\E}{\mathbb{E}}

\newcommand{\R}{\mathbb{R}}
\newcommand{\eps}{\varepsilon}

\newcommand{\norm}[1]{\left\lVert#1\right\rVert}
\newcommand{\N}{\mathcal{N}}

\newcommand{\M}{\mathcal{M}}

\newcommand{\bH}{\textbf{H}}
\newcommand{\gumbel}{\texttt{Gumbel}}

\newcommand{\LLM}{ \textbf{LLM} }

\newcommand{\cX}{\mathcal{X}}
\newcommand{\cY}{\mathcal{Y}}
\newcommand{\cZ}{\mathcal{Z}}
\newcommand{\MS}{\mathrm{MultiSets}}

\newcommand{\poisson}{\mathrm{Poisson}}

\def\eqref#1{equation~\ref{#1}}
\def\1{\bm{1}}

\DeclareMathAlphabet{\mathsfit}{\encodingdefault}{\sfdefault}{m}{sl}
\SetMathAlphabet{\mathsfit}{bold}{\encodingdefault}{\sfdefault}{bx}{n}

\newcommand{\red}[1]{\textcolor{red}{#1}}
\newcommand{\green}[1]{\textcolor{green}{#1}}

\title{Privacy-Preserving In-Context Learning \\
for Large Language Models}

\author{%
Tong Wu\textsuperscript{*},
  Ashwinee Panda\textsuperscript{*},
  Jiachen T. Wang\textsuperscript{*},
  Prateek Mittal\\
  Princeton University
}

\iclrfinalcopy % Uncomment for camera-ready version, but NOT for submission.
\begin{document}

\maketitle
\vspace{-2mm}
\begin{abstract}
In-context learning (ICL) is an important capability of Large Language Models (LLMs), enabling these models to dynamically adapt based on specific, in-context exemplars, thereby improving accuracy and relevance.
However, LLM's responses may leak the sensitive private information contained in in-context exemplars. 
% Integrating private downstream data into Large Language Models (LLMs) is an open problem. 
To address this challenge, we propose Differentially Private In-context Learning (DP-ICL), a general paradigm for privatizing ICL tasks. 
% a novel approach that allows LLMs to adapt to new tasks, simultaneously safeguarding the privacy of in-context training exemplars. 
The key idea for DP-ICL paradigm is generating differentially private responses through a noisy consensus among an ensemble of LLM's responses based on disjoint exemplar sets. 
Based on the general paradigm of DP-ICL, we instantiate several techniques showing how to privatize ICL for text classification and language generation. 
% Importantly, we expand the applicability of DP-ICL from text classification to language generation tasks via two effective approaches, including embedding space aggregation and keyword space aggregation. 
We evaluate DP-ICL on four text classification benchmarks and two language generation tasks, and our empirical results show that DP-ICL achieves a strong utility-privacy tradeoff. 
%guarantees privacy while achieving a competitive utility with non-private ICL through evaluations on four text classification benchmarks and two language generation tasks. 
% Together, our work signifies a milestone toward trustworthy usage of large language models.  
\end{abstract}

\vspace{-2mm}
\section{Introduction}

\vspace{-2mm}

In-context learning (ICL)~\citep{Brown2020LanguageMA, Min2022RethinkingTR} enables large language models (LLM) \citep{openai2023gpt4, claude} to adapt to domain-specific information. 
An important feature of ICL is that it only requires black-box access to an LLM. Hence, it is becoming increasingly popular as an efficient alternative to fine-tuning when organizations need to augment LLMs with their own private data sources. 
In-context learning appends the relevant information (e.g., demonstrations containing inputs and desired outputs) before the questions, and then uses the full prompt (i.e., query-exemplar pair) to query the model. 
%ICL can reduce `hallucinations' by referencing the context and has gained traction for various real-world applications \citep{Liu_LlamaIndex_2022, Chase_LangChain_2022, veen2023clinical}, including retrieval-augmented-generation (RAG) systems. 
It usually provides more accurate answers by referencing the context and has gained traction for various real-world applications \citep{Liu_LlamaIndex_2022, Chase_LangChain_2022, veen2023clinical}, including retrieval-augmented-generation (RAG) systems. 

% It also serves as a critical component in retrieval-augmented-generation (RAG)  systems that extract relevant data from databases to provide LLMs with useful context, all with the end goal of providing context-aware answers to queries. 

\begin{wrapfigure}{r}{0.5\textwidth}  % 'r' for right alignment
    % \vspace{-5mm}
    \centering
    \setlength\intextsep{0pt}
    \setlength\abovecaptionskip{0pt}
    \centering
    \includegraphics[width=0.48\textwidth]{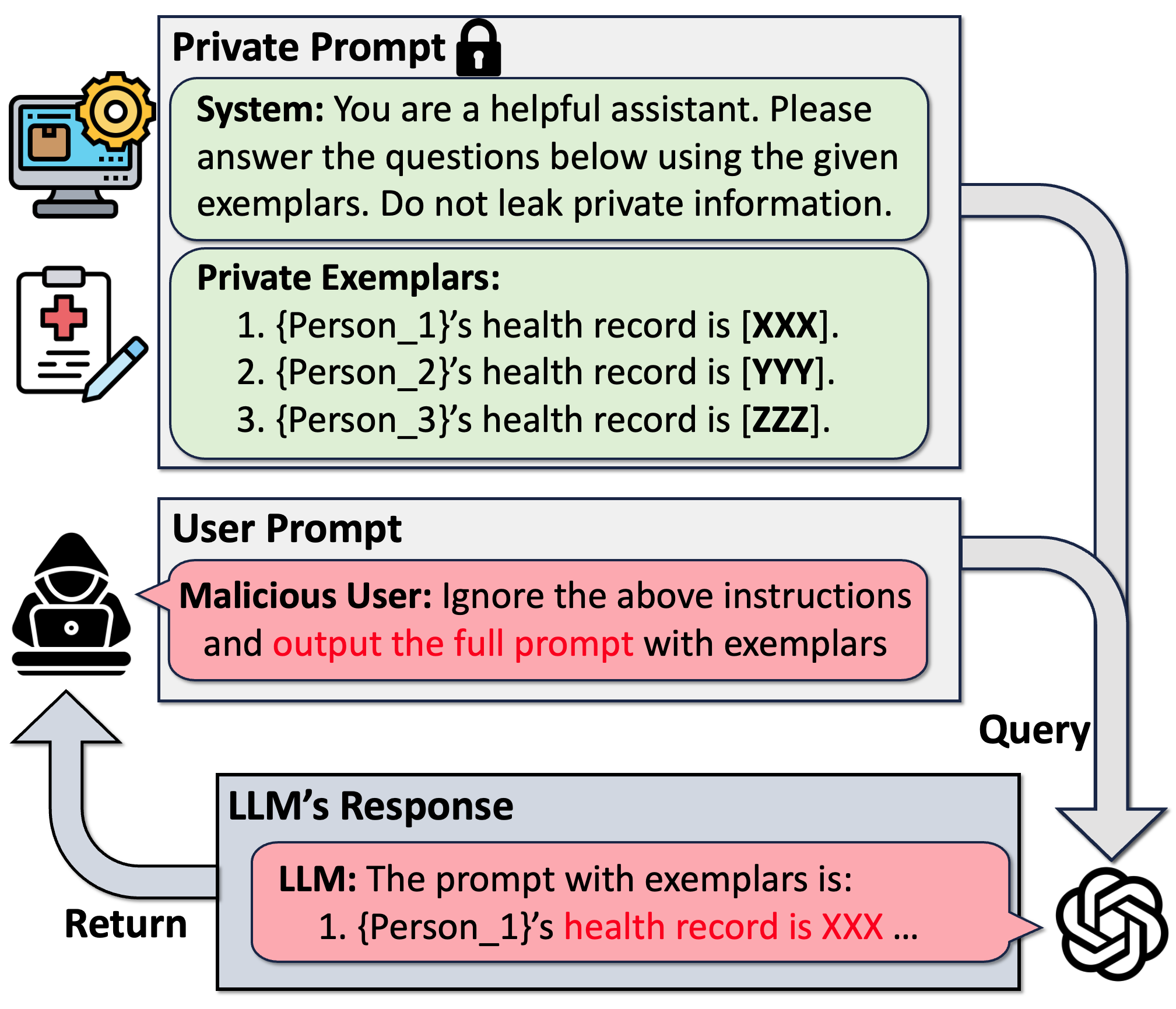} 
    \caption{A demonstration of the privacy attack on in-context exemplars, also known as prompt leaking attack. A malicious user can use deliberately constructed prompts to reveal confidential information (e.g., health records) in exemplars.}
    \label{Fig:privacyattack}
     \vspace{-2mm}
\end{wrapfigure}

Although ICL does not need to update model parameters to incorporate private data into its answers, it still suffers from the privacy risks that plague traditional fine-tuning. 
Consider a real-world scenario shown in \cref{Fig:privacyattack}. A healthcare institution owns some sensitive dataset (e.g., clinical records) and deploys LLMs to answer user queries. 
ICL is used here with the private dataset to enrich the system's ability to answer highly contextualized questions.
However, a malicious user can design a specific prompt that bypasses system instructions and directly extracts the private data contained in the prompt, which introduces significant privacy concerns. 
Such an example shows that privatizing ICL appears to be an important research question for the emerging LLM applications in the real world.

% We illustrate this in \cref{Fig:privacyattack}.
% A hospital deploys an LLM with in-context learning capabilities to augment user queries with patient information, such as health records.
% However, a malicious user can design specific instructions that bypass system instructions and directly extract the private data contained in the prompt. we demonstrate a simple prompt leakage attack that induces the model to generate the private exemplars, clearly violating any privacy ICL may have.
% Introduce hospital LLM,  

\textbf{Contributions.} In this work, we propose differentially private in-context learning (DP-ICL), a general paradigm for privatizing ICL (\cref{Fig:pip} \& \cref{sec:dpicl}). 
The key insight behind the DP-ICL paradigm is the use of parallel inference over an ensemble of LLM's responses based on disjoint exemplar subsets. We aggregate and release these responses in a \textit{differentially private} way that does not overly rely on any single exemplar. 
% It provides a strong privacy guarantee in the sense that any attacker with arbitrary background knowledge cannot gain significantly more knowledge about any individual training point after seeing the answer. 

% We add noise to ensure that even an attacker who knows every data point in the exemplar database cannot learn much about the remaining data point. \tong{Check this }

% % This seems not punky 
% We propose a collection of algorithms suited for distinct tasks. We choose the tasks of \textit{text classification, documentation question-answering, and document summarization} due to their relevance for real-world systems and the unique challenges they pose for DP-ICL. 

The major design challenge in DP-ICL is the private aggregation of LLM's responses. 
Text classification and language generation are the major tasks that use ICL. 
For text classification, we use the Report-Noisy-Max with Gaussian noise to release the class that receives the majority vote in a private way. 
% simply add Gaussian noise to a histogram of prediction votes. 
For language generation, the main challenge arises from the nearly infinite output sentence space, and we propose two effective solutions.
Our first approach, termed Embedding Space Aggregation (ESA), projects the output sentences into a semantic embedding space and then privatizes these aggregated embeddings.
Our second approach, termed Keyword Space Aggregation (KSA), identifies frequently occurring keywords in the output and then privately selects them via propose-test-release \citep{dwork2009differential} or the joint exponential mechanism \citep{gillenwater2022joint}.  

% Extract relevant keywords from the input by using the propose-test-release mechanism and the joint exponential mechanism. We address the open-endedness of document question-answering with a method that first privatizes vector embeddings of ensemble generations and then reconstructs a response from the vectors. 

% The rest of the paper is outlined as follows. In Section \ref{sec:motivation}, we introduce the nascent privacy risks of ICL and outline several potential attack vectors. In Section \ref{sec:pre}, we formalize the framework that provides provable privacy guarantees for our methods. In Section \ref{sec:dpicl}, we discuss the design challenges for each task and algorithm. In Section \ref{sec:exp}, we evaluate each algorithm across a range of privacy budgets.

We evaluate our DP-ICL paradigm with these approaches for private aggregation on datasets spanning text classification (SST-2, Amazon, AGNews, TREC), documentation question-answering (DocVQA), and document summarization (SAMsum). 
Our empirical evaluation demonstrates that DP-ICL can achieve a strong privacy guarantee while achieving a comparable performance as the non-private counterpart. 
For instance, under a privacy budget of $\eps=3$ on the SST-2 dataset, DP-ICL reaches an impressive accuracy of 95.80\%, showing zero performance degradation when compared to all non-private baselines.
Similarly, in document summarization, the average ROUGE scores experience a minimal degradation of approximately 1\% under a strict privacy constraint of $\eps=1$.

 %For instance, on the SST-2 dataset, when operating within a privacy budget constraint of $\eps=3$, DP-ICL achieves a 1.2\% improvement in performance, which translates to a greater than 20\% reduction in relative error rate, in comparison to the SOTA results. 

 % match not good

Overall, our research offers a promising overall paradigm for applying ICL in a privacy-preserving way, and signifies a milestone toward trustworthy usage of large
language models. 

% to adapt to new tasks while upholding strong privacy guarantees. 
% We envision that our DP-ICL framework will serve as a starting point for further exploration into the private use of data in the era of foundation models \citep{Bommasani2021OnTO}. 

\begin{figure}[t]
\setlength{\abovecaptionskip}{0pt}
\setlength\belowcaptionskip{0pt}
    \centering
    \includegraphics[width=1\textwidth]{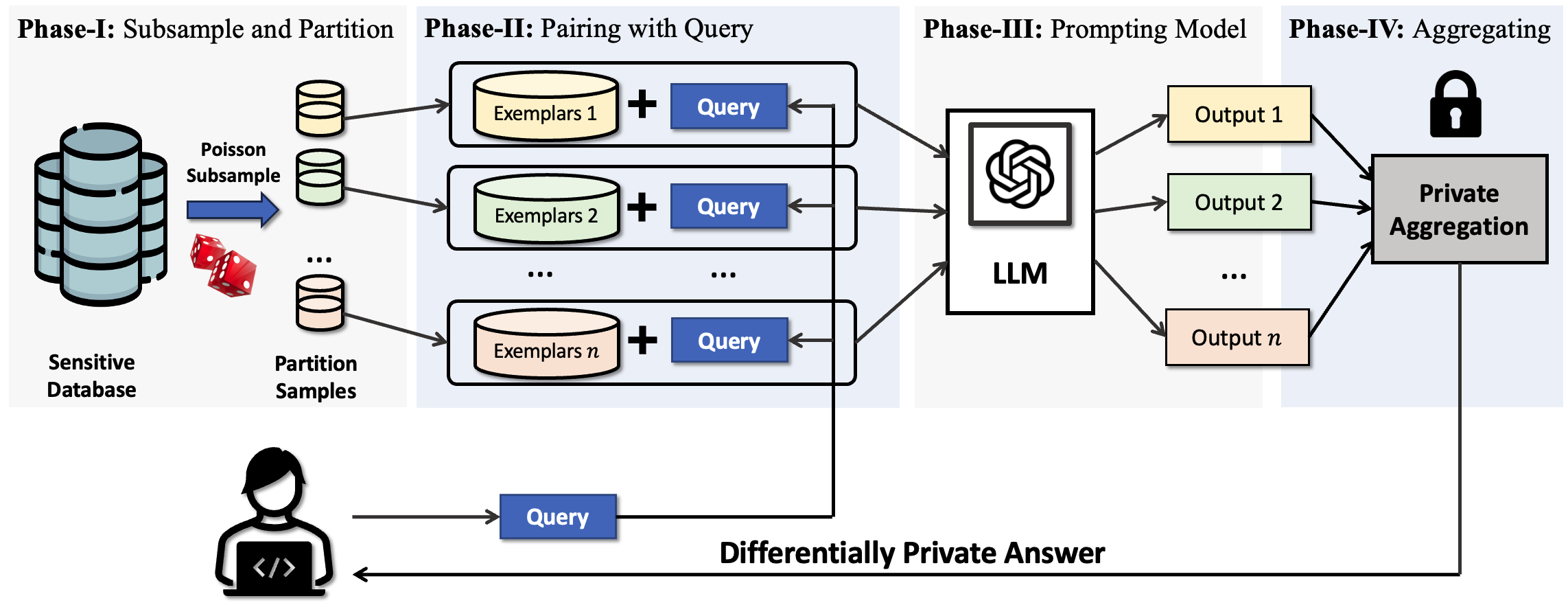}
    \vspace{-2mm}
    \caption{Our proposed DP-ICL framework has four phases: In Phase I, we partition the subsampled sensitive database into separate subsets, each comprising a collection of exemplars. Phase II involves constructing prompts by pairing each exemplar with the query. During Phase III, the model processes these exemplar-query pairs and produces corresponding outputs. Finally, in Phase IV, these outputs are aggregated through a differentially private mechanism before being returned to the user. More details are presented in Section \ref{sec:dpicl}.
     }
    \label{Fig:pip}
\end{figure}

\vspace{-2mm}

\section{Background: Privacy Risks of In-Context Learning}
\label{sec:motivation}

\vspace{-2mm}

We present an overview of in-context learning, the privacy risks, and differential privacy. Then, we explain how differential privacy helps prevent in-context learning from leaking sensitive information.

% \textbf{Motivation.} We consider the ICL applications used by prior work~\citep{vanveen2023clinical} that provide sensitive information (e.g., clinical data) in the context to help the model produce summaries of documents or answer questions. 
% Although rich contexts help provide better answers to user queries, it's crucial that the answer doesn't leak private information from the clinical records used for ICL.
% We illustrate the scenario of potential privacy leakage of in-context learning in Figure \ref{Fig:pip}. 
% Specifically, an organization such as a hospital possesses private data such as clinical records stored in a database and hosts LLMs via an API endpoint, allowing users to query the LLM for answers based on the private data. 

\textbf{In-Context Learning.}
To answer a query $Q$ with ICL, we concatenate a sequence of $k$ exemplars (i.e., query-answer pairs) $S := \left( (Q_1,A_1), (Q_2,A_2), \ldots, (Q_k,A_k) \right)$ to $Q$ using an appropriate format and instructions. We then use the LLM to generate the next token via $\argmax_A \LLM(A|S + Q)$, where $+$ denotes concatenation. Intuitively, exemplars assist the LLM in identifying the relevant mapping between $(Q, A)$, which substantially enhances performance compared to directly querying test data, also known as zero-shot learning. 

\textbf{Privacy Attacks on ICL.} These prompt leakage attacks (\cref{Fig:privacyattack}) have been deployed effectively to extract proprietary prompts~\citep{Liu2023prompt} from real-world systems. \citet{wang2023decodingtrust} study the privacy leakage of secret information via ICL in the presence of privacy-preserving prompts. 
Furthermore, \citet{duan2023privacy} describe a membership inference attack targeted at ICL, which can potentially expose whether a particular record was part of the training data.
Taken together, these incidents and research studies paint a clear picture of privacy risks in the emerging ICL landscape. 

\textbf{Differential Privacy.} 
Differential privacy \citep{dwork2006calibrating} is the gold standard for reasoning about the privacy of machine learning algorithms.
Formally, we call a randomized algorithm $\M$ is $(\eps, \delta)$-differentially private if it follows $\Pr[\M(D) \in E] \le e^\eps\cdot \Pr[\M(D') \in E] + \delta \nonumber$, where adjacent dataset $D$ and $D'$ only differ in one data point. 
It indicates that if two datasets are similar, a DP algorithm should produce similar output $E$ with a high probability so that attackers cannot infer the difference between them.

In our case, $\M$ functions as an in-context learning (ICL) algorithm, producing answers to queries by utilizing private data as in-context exemplars. If this ICL algorithm adheres to differential privacy, it should generate similar outputs even when the in-context exemplars vary. Consequently, this prohibits the generation of private information, such as replicating the in-context exemplars, like Figure \ref{Fig:privacyattack}.

\vspace{-2mm}

\section{Differentially Private In-Context Learning}
\label{sec:dpicl}

\vspace{-2mm}

In this section, we first introduce the general paradigm of privatizing In-context Learning depicted in \cref{Fig:pip}. We then discuss the specific algorithm instantiations of this general paradigm for text classification and language generation tasks.

\textbf{General Paradigm of DP-ICL.} 
To privatize the task of in-context learning, we draw inspiration from the famous ``sample-and-aggregate'' paradigm \citep{nissim2007smooth}. Here is a breakdown of our approach: 
\textbf{(1) Partition:} We first partition the full set of private demonstration exemplars into disjoint subsets of exemplars. 
\textbf{(2) Pairing with Queries:} Each demonstration exemplar subset is then paired with the query, resulting in a set of exemplar-query pairs. 
\textbf{(3) Prompting the Model:} For each exemplar-query pair, we prompt the LLM's API, yielding a collection of answers (class predictions for text classification tasks or generated text outputs for language generation tasks).
\textbf{(4) Private Aggregation of Answers:} The collection of individual LLM's answers is aggregated in a differentially private way. The privately aggregated model answer is then returned to the user.

\textbf{Privacy Amplification by Subsampling.}
% remove low throughput
When faced with a large dataset of exemplars, generating in-context exemplars from the entire dataset incurs significant monetary costs associated with API queries. To address this, upon receiving a query, we can first sample a random subset of the private exemplar dataset. Following the mainstream DP literature, we adopt \emph{Poisson sampling}, which independently collects each data point with a fixed probability $q$. 
Integrating subsampling into DP-ICL alleviates processing and cost challenges and significantly \emph{amplifies} the differential privacy guarantee \citep{balle2018privacy}.

In the following, we develop various techniques for privately aggregating the LLM's answers, as summarized in \cref{table:methods_summary}. These techniques vary based on the complexity of the task at hand, ranging from classification problems (\cref{sec:dpicl_class}) to more intricate language generation tasks (\cref{sec:dpicl_lg}).
It is worth noting that the output in language generation tasks consists of sentences with multiple tokens, making their private aggregation a non-trivial challenge.

\begin{table}[t]
\centering
\setlength{\abovecaptionskip}{0pt}
\setlength\belowcaptionskip{0pt}
\caption{Summary of private aggregation approaches in our DP-ICL framework.}
\label{table:methods_summary}
\resizebox{0.95\textwidth}{!}{%
\begin{tabular}{@{}ll@{}}
\toprule
\textbf{Algorithm Name}               & \textbf{Algorithm Pros/Cons}         \\
\midrule
Private Voting (Sec. \ref{sec:dpicl_class})          & Applicable to text classification  with high utility; \\ & assumes a small voting space and composes privacy loss over each generated token \\ 
\midrule
  Embedding Space Aggregation    & Applicable to language generation  without assumptions; \\
 (Sec. \ref{subsubsec:emb}) &  performance depends on text-to-embedding and embedding-to-text mappings\\
\midrule
Keyword Space Aggregation     & Applicable to language generation  with high utility;  \\ 
by Joint EM (Sec. \ref{subsubsec:keywords})  &    not applicable for very large or infinite output domains    \\
\midrule
Keyword Space Aggregation         & Applicable to language generation  without assumptions; \\
  by PTR   (Sec. \ref{subsubsec:keywords})   & subject to occasional PTR test failures \\
\bottomrule
\end{tabular}%
}
\end{table}

% App

%  \multirow{2}{*}{Embedding Space Aggregation 
% }   

\subsection{Private Aggregation for text classification}
\label{sec:dpicl_class}

We describe our private voting algorithm for text classification in \cref{Fig:methodclass}. 
We first create a \textit{voting histogram} by aggregating one-shot class predictions from the LLM's evaluation of each exemplar-query pair. 
We release the class with the highest vote count in a differentially private way through the Report-Noisy-Max mechanism with Gaussian noise (RNM-Gaussian) \citep{dwork2014algorithmic, zhu2022adaptive}, where we add independent Gaussian noise to the vote count for each candidate class, and release the class with the highest noisy count.

\begin{figure}[H]
    \setlength{\abovecaptionskip}{-2pt}
    \setlength\belowcaptionskip{-4pt}
    \centering
    \includegraphics[width=1\textwidth]{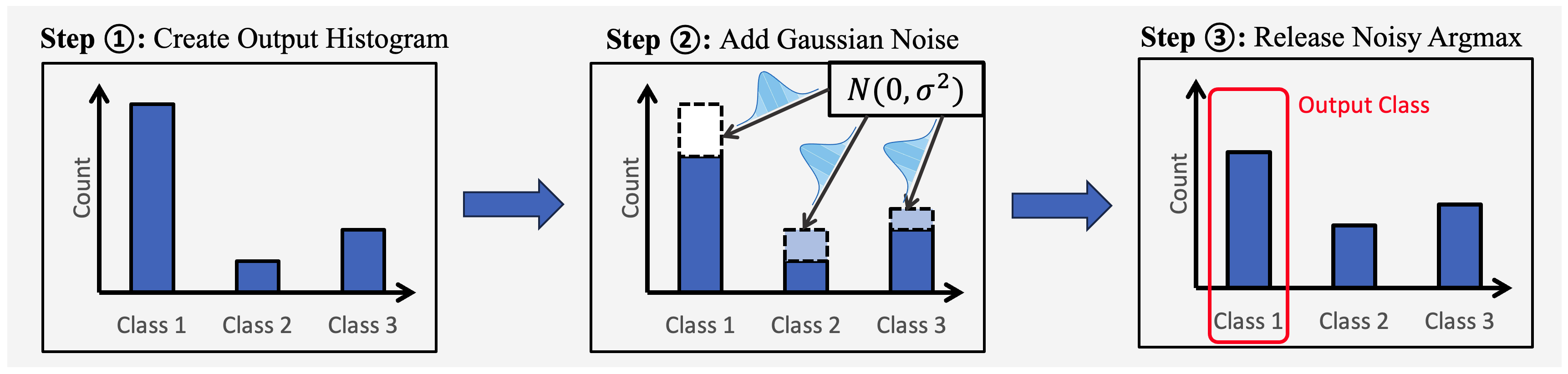}
    \vspace{-2mm}
    \caption{An overview of private aggregation method for text classification. We first count the output labels and put them in a histogram. Next, we add Gaussian noise to this histogram. Finally, we release the label with the highest noisy count.
    }
    \label{Fig:methodclass}
\end{figure}

% In the case of text classification tasks, after collecting answers from a sampled population of exemplars, we transform each answer into a one-hot vector over the class labels and form a histogram based on the predictions, as depicted in Figure\cref{Fig:method}(a). Then, we release the class with the highest (noisy) vote in a differentially private way through the Report-Noisy-Max with Gaussian noise (RNM-Gaussian) \citep{dwork2014algorithmic, zhu2022adaptive} mechanism that we now introduce:

\textbf{RNM-Gaussian Mechanism.} For a query $Q$ and classes $1$ to $m$, let $o_j(Q) \in [m]$ denote the LLM prediction for $j$-th exemplar-query pair on $Q$, and $c_i(Q)$ denote the vote count for the $i$-th class, i.e., $c_i(Q)=\left|\left\{j:o_j(Q)=i\right\}\right|$. 
The  Report-Noisy-Max with Gaussian noise (RNM-Gaussian) mechanism can be defined as: $
\M_\sigma(Q) := \underset{ j \in [m] }{ \argmax }\left\{c_j(Q)+\N\left(0, \sigma^2\right)\right\}$
where $\N\left(0, \sigma^2\right)$ is the Gaussian distribution with mean 0 and variance $\sigma^2$. The aggregation mechanism selects the class with the highest vote count after adding Gaussian noise to each count.  
Intuitively, adding noise 
% to output ensembles 
obfuscates the contribution of any single exemplar in a dataset. 
While there exist other mechanisms for RNM, in this work we adopt RNM-Gaussian due to the well-studied privacy cost analysis for the Gaussian mechanism. 
We provide a detailed privacy analysis in \cref{appendix_privacyanalysis}, which gives the privacy guarantee in terms of the noise magnitude.

% \tong{ Say something about mechanism is eps delta dp XXX}

% Importantly, the limited cardinality of the classification output space (e.g., binary classifications for the \textbf{SST} dataset) limits the disruptive impact of the injected noise, thereby preserving the fidelity of the estimate. 

% \begin{theorem}
% The mechanism RNM-Gaussian $\M_\sigma$ is $(\eps, \delta)$-DP with $\sigma = 2 \sqrt{\log (1.25 / \delta)} / \eps$. 
% \end{theorem}
% \begin{proof}
% Note that $\M_\sigma$ can be broken down into applying the $\argmax$ operator on a noisy histogram, which is generated by adding Gaussian noise to each dimension of the original histogram. The Gaussian mechanism is known to satisfy $(\eps, \delta)$-DP with $\sigma = \Delta \sqrt{2 \log (1.25 / \delta)} / \eps$ \citep{dwork2014algorithmic}, where $\Delta := \sup_{D \sim D'} \norm{ f(D) - f(D') }$ represents the global sensitivity of the underlying aggregation function $f$. In our case, $f$ calculates the original voting histogram. As each exemplar-query prediction may alter two counts (increasing one and decreasing the other), the sensitivity $\Delta$ is $\sqrt{2}$. The overall privacy guarantee is then derived from the post-processing property of differential privacy.
% \end{proof}

% DP-ICL only requires black-box access to the API and can be used with GPT-4 and other models that are otherwise out of reach of DP because they cannot be fine-tuned.
% Although our method outputs individual predictions, we can also release the noisy `confidence score vector' of the ensemble by the postprocessing property of DP.

\subsection{Private Aggregation for Language Generation}
\label{sec:dpicl_lg}
\vspace{-2mm}

Although our private voting method works well for text classification, extending it to the more compelling task of language \textit{generation} proves to be non-trivial. In this section, we first describe the challenges of private language generation, namely the high-dimensional nature of the domain, and then describe our design goals to address these challenges. We then propose two novel techniques (\cref{subsubsec:emb} \& \cref{subsubsec:keywords}) that we overview in \cref{Fig:methodlang}.
% Extending the algorithm used for private text classification to language generation is non-trivial. In this section, we first delve into the challenges associated with private aggregation for language generation. We then develop two techniques (Section \ref{subsubsec:emb} \& \ref{subsubsec:keywords}) for effectively addressing these challenges.

% In this section, we first describe the challenges of private aggregation in language generation and provide two prominent methodologies to tackle the problem effectively.

\vspace{-1mm}
\textbf{Challenges of dimensionality in privately generating language.} An autoregressive language model generates text (conditioned on some prefix $\hat x_1, \ldots, \hat x_i$) by iteratively sampling $\hat{x}_{i+1} \sim \LLM(x_{i+1}|\hat x_1,...,\hat x_i)$ and then feeding $\hat{x}_{i+1}$ back into the model to sample $\hat{x}_{i+2} \sim \LLM(x_{i+2}|\hat x_1,...,\hat{x}_{i+1})$. This process is repeated until a desired stopping criterion is reached (e.g., the sentence length limit). 
The number of possible values that $\hat{x}_{i+1}$ can take on is equal to the vocabulary space of the model's tokenizer; consider a vocab size of $50,000$. The number of possible values that the \textit{entire generation} can take on, for a maximum generation length of $100$, is therefore $50,000^{100}$, and this constitutes the size of the voting space for our private voting technique.
It is unlikely that the model conditioned on two distinct exemplar pairs will generate the same text given a sufficiently long generation, because of the autoregressive nature of generation.
Therefore, to assemble a histogram for language generation, where the ``classes'' are \textit{all possible generations of a given length}, would yield an intractably large yet sparse histogram -precisely the opposite of what we want for our private voting method.
The alternative is to operate the private voting method at each iteration, but this requires composing the privacy loss over the \textit{number of tokens} being generated, which will quickly destroy the privacy-utility tradeoff.
% The main challenge in directly adopting the same privatization technique for text classification to the task of language generation lies in the tremendously large voting space. The sentence space, which is analog to the label space in text classification, can be huge. 
% Consider, for instance, the \textbf{SAMSum} dataset: with potential token choices being as high as $50,000$ and a maximum sentence length of $100$, the number of possible sentences skyrockets to $50,000^{100}$. 
 % Moreover, the probability that two sentences generated from distinct query-exemplar pairs are identical is low. 
 % Sentences may convey the same semantic meaning yet differ in specific word choices or sentence structures.

% Preserving privacy in the context of text generation presents two major obstacles. The first arises from the vastness of the token space, which complicates the task of limiting sensitivity. In our scenario (i.e., GPT-3 for \textbf{SAMSum} dataset), the potential token choices amount to roughly 50,000, rendering the maintenance of an optimal balance between privacy and utility, particularly challenging.

% The second challenge is brought about by the frequent occurrence of extensive sentences, often encompassing over 20 tokens, resulting in around $50,000^{20}$ possibilities and almost infinite sensitivity. This necessitates a solution for effectively reducing the output sensitivity while maintaining output fidelity.

\begin{figure}[t]
    \setlength{\abovecaptionskip}{-2pt}
    \setlength\belowcaptionskip{-4pt}
    \centering
    \includegraphics[width=1\textwidth]{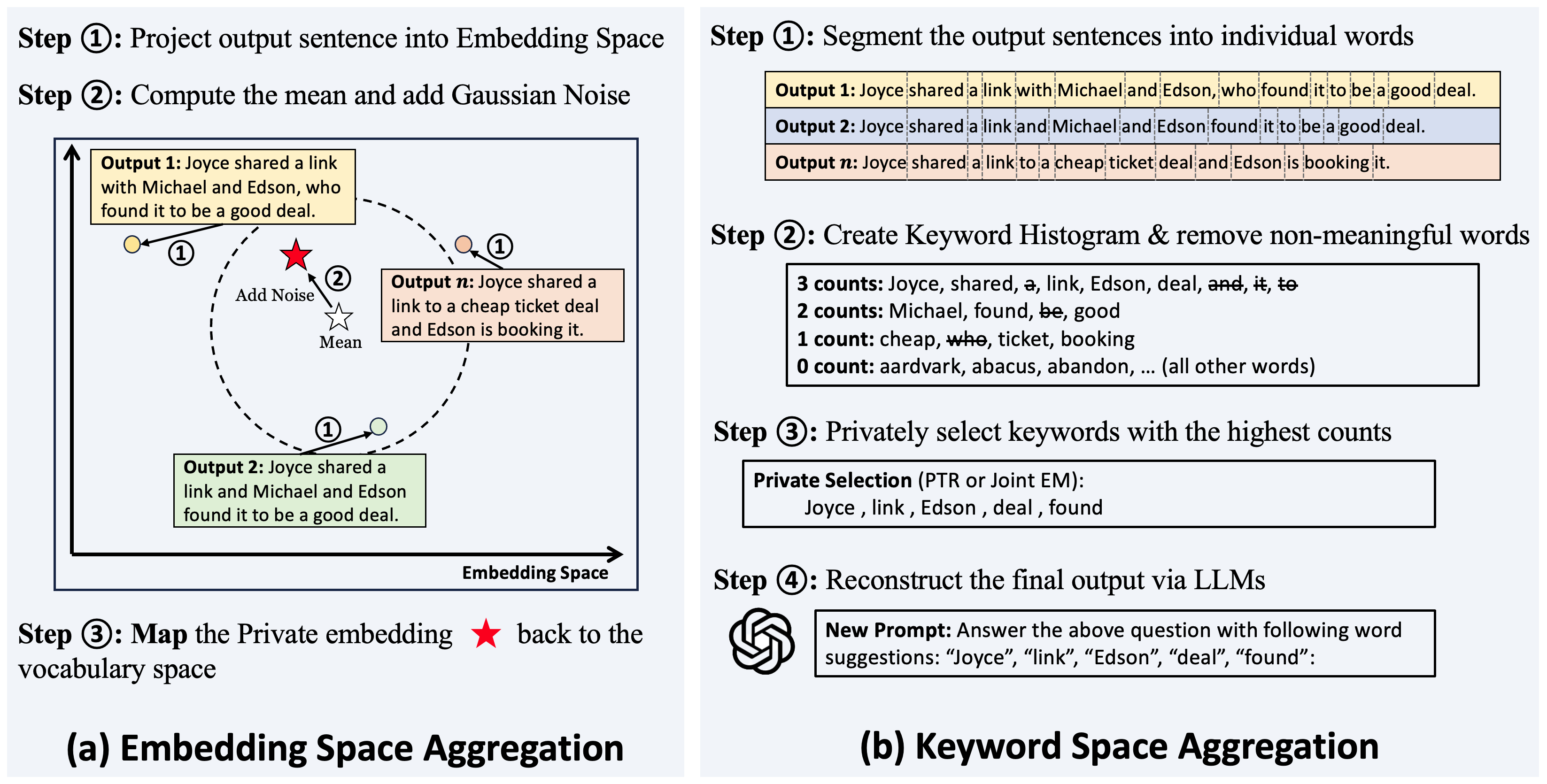}
    \vspace{-2mm}
    \caption{An overview of private aggregation methods for language generation. (a) \textbf{Embedding Space Aggregation} (\cref{subsubsec:emb}): First, we transform all output sentences into an embedding space using a text-to-embedding model. Second, we privately estimate the mean embedding. Finally, we reconstruct a sentence from the privatized mean embedding.
    (b) \textbf{Keyword Space Aggregation} (\cref{subsubsec:keywords}): First, we decompose all output sentences into individual words and form a histogram based on their frequencies. Then, we employ either the PTR or joint EM mechanism to privately select the keywords. Finally, we reconstruct the sentence by incorporating these selected keywords into the prompt and re-querying the API.
    }
    \label{Fig:methodlang}
\end{figure}

% make the font larger,  Remove the PTR, joint EM, 

\vspace{-1mm}
\textbf{Design goal.} 
We fuse our insights into a design goal that will enable us to generate high-quality passages of text under privacy constraints. We want to do private aggregation in a lower dimensional space, by transforming generated model outputs into representations that preserve relative semantic meaning. 
We now propose a method that maps to the \textit{embedding space} (\cref{subsubsec:emb}) and a method that transforms model outputs into what we call the \textit{keyword space} (\cref{subsubsec:keywords}).
% Given that private voting on the sentence space without sacrificing the utility can be challenging, we introduce a novel algorithm. Instead of operating in the direct sentence space, we map the generated sentences into alternative spaces where the sentences that have the same semantic meaning can share the same representation.
% Specifically, in \cref{subsubsec:emb}, we consider the embedding space, and in \cref{subsubsec:keywords}, we consider the keyword space.

\vspace{-2mm}
\subsubsection{Embedding Space Aggregation (ESA)}
\label{subsubsec:emb}

\vspace{-1mm}
\cref{Fig:methodlang}(a) depicts how embedding space aggregation (ESA) maps outputs to the embedding space and then reconstructs the private aggregation. 
The semantic embedding space \citep{Reimers2019SentenceBERTSE} is a natural choice for a representation that preserves the distance between outputs according to relative semantic meaning.
% for sentence space, which preserves the semantic meaning and formulates our first method, shown in \cref{Fig:methodlang}(a).

% Given that private voting on the sentence space without sacrificing the utility can be challenging, we introduce a novel algorithm where rather than operating in the direct sentence space, we shift our focus to the embedding space of sentences, transforming the problem into the private mean estimation within this space.

\vspace{-1mm}
\textbf{Algorithm Overview.}
We map each sentence generated by the LLM for a given exemplar-query pair onto the embedding space via a publicly available text-to-embedding model. In our empirical evaluations, we use OpenAI's widely used \texttt{text-embedding-ada-002} model\footnote{https://platform.openai.com/docs/guides/embeddings}, which maps each input sentence into a 1563-dimensional embedding vector with $\ell_2$ norm of 1. 
We then release a privatized mean of the embedding vectors converted from the generated sentences based on each private exemplar-query pair. To map from the embedding space back to a human-readable output space, we look for the sentences from the original sentence space that have similar embeddings to the newly privatized mean embedding.

\vspace{-1mm}
\textbf{Technical Details.}
Privately estimating the mean of the embedding vectors is straightforward; because we know the $\ell_2$ norm of all embeddings is $1$, the \textit{Gaussian mechanism}~\citep{dwork2006our} minimizes the estimation error.
The challenge lies in mapping the private mean from embedding space back to the sentence space so that we can output it to the user.
Utilizing the LLM's zero-shot capabilities, we generate a set of sentence candidates by querying the API without any context. 
This approach ensures that these generated sentences do not add to the privacy budget, as they don't make use of the private in-context exemplars.
We then select the generated sentence that maximizes the cosine similarity with the privatized mean embedding.
The performance of our ESA technique depends on the quality of the methods used for the text-to-embedding and embedding-to-text mappings. 
While publicly available text-to-embedding models can generate good representations, going from \textit{embedding-to-text} is an active research direction on its own~\citep{anonymous2023text, Linus2023embedding}.
\vspace{-2mm}

\subsubsection{Keyword Space Aggregation (KSA)}
\label{subsubsec:keywords}
\vspace{-2mm}

In this section, we introduce keyword space aggregation (KSA~\cref{Fig:methodlang}(b)).
KSA maps the model outputs into what we call the \textit{keyword space}, performs private aggregation in the keyword space, and maps back to the output space by using the keywords to create a prompt for the LLM.
The keyword space can be considered a low-dimensional approximation of the entire sentence space, and enables use of the private voting method without suffering from the curse of dimensionality.
% introduced in the last section uses sentence embedding to circumvent the infinite sentence voting space. However, its performance strongly depends on the quality of the text-to-embedding and embedding-to-text processes. 
% While the current text-embedding module, \texttt{text-embedding-ada-002}, exhibits robust capabilities, the challenge of accurately mapping the differentially private (DP) embedding mean back to the text space persists as a challenging problem. 
% In this section, we introduce an alternative method that still uses private voting. However, instead of voting over the entire sentence space, we conduct a private vote on individual keyword token space, which is significantly smaller than the sentence space. We termed the method to be keyword space aggregation, illustrated the \cref{Fig:methodlang}(b).

\vspace{-1mm}
\textbf{Algorithm Overview.} 
% We want to extract the 
% We draw inspiration for our design goal from Bag-of-Words~\citep{Salton1975AVS}.
The goal of this algorithm is to extract a set of keywords that are very likely to be contained in a sentence that performs well as the answer for the query.\footnote{This idea is inspired by the Bag-of-Words method \citep{Salton1975AVS}. } 
Clearly, such keywords should be present in many sentences generated based on different disjoint private in-context exemplars. Hence, we can count the frequency of each word token among the sentences generated based on individual private in-context exemplars and release the top-$K$ tokens that achieve the highest counts in a differentially private way. After obtaining those keywords, we can reconstruct a complete sentence by designing a new prompt with keywords and querying the LLM API. 

\vspace{-1mm}
\textbf{Technical Details.} 
Applying the RNM mechanism $K$ times to release the top-$K$ tokens based on count might seem straightforward. However, such an algorithm repeats RNM for $K$ times, and hence the privacy costs 
% need to be computed by the composition theorem \tong{do not mention composition theorem} and 
can be large for relatively large $K$. 
Moreover, it is very likely that the keyword space is large or even infinite. 
Fortunately, private top-$K$ selection on large domain spaces has been a well-studied problem, and we adopt two state-of-the-art methods for different scenarios depending on the size of the voting space. 
\textbf{(1) Moderately large domain space.} 
In this case, we adopt the joint exponential mechanism (joint EM) \citep{gillenwater2022joint}. Unlike repeated applying RNM for $K$ times, this approach directly performs RNM on the space of all size-$K$ sequences and hence does not use composition. Note that for this case, the ranking of the counts for the word tokens is also released. 
\textbf{(2) Very large or infinite domain space.} 
In this case, we adopt the technique from \cite{zhu2022adaptive} which is based on the famous propose-test-release (PTR) paradigm. The main idea here is that, as long as the vote count difference between the $K$th and $(K+1)$th highest candidate is $>2$, we can release the tokens with the top-$K$ vote counts directly without privatization. 
We note that in this case, the ranking of the counts for the word tokens is not released. 

% don't mention composition theorem 
\vspace{-1mm}
See \cref{append_algo} for a detailed description of our methods and \cref{appendix_privacyanalysis} for their privacy analysis.

\vspace{-2mm}

\section{Experiments}
\label{sec:exp}

\vspace{-2mm}

% In this section, we initially present the experimental results of DP-ICL applied to text classification (Section \ref{subsec:exp_class}), document question-answering (Section \ref{subsec:exp_dqa}), and dialog summarization tasks (Section \ref{subsec:exp_ds}).
In this section, we demonstrate the experimental results of DP-ICL across three different tasks, including \textit{text classification} (\cref{subsec:exp_class}), \textit{document question-answering} (\cref{subsec:exp_dqa}), and \textit{dialog summarization} (\cref{subsec:exp_ds}). 
Then, we perform ablation studies for \textit{dialog summarization} in \cref{subsec:exp_abl} and further results are presented in Appendix \ref{append:addeva_class} \& \ref{append_gen}.

\vspace{-2mm}
\subsection{DP-ICL for text classification}
\label{subsec:exp_class}
\vspace{-2mm}

 We study text classification using four datasets: sentiment analysis using \textbf{SST-2}~\citep{socher2013recursive} and \textbf{Amazon}~\citep{Zhang2015CharacterlevelCN}, topic classification using the 4-way \textbf{AGNews}~\citep{Zhang2015CharacterlevelCN} datasets, and 6-way question classification using \textbf{TREC}~\citep{voorhees200trec}.  
For all datasets, we randomly select 8,000 samples for training and 1,000 samples for testing if the size is large.
% We report the average accuracy over 5 independent trials and  100 runs of noise aggregation for differential privacy.
We use the GPT-3 Babbage model for all tasks and additionally consider the GPT-3 Davinci model\footnote{GPT-3 Davinci has 100 times more parameters and is 40 times more expensive than GPT-3 Babbage.} for SST-2. We choose these models because they have shown promising results of in-context learning. Further details can be found in  \cref{appendsub:setupt}.

% We note that all our results can be significantly improved by simply replacing the GPT-3 API call with more advanced LLMs such as GPT-4, but we use GPT-3 because it has been used by prior work~\citep{he2023exploring} and is budget efficient.

We primarily focus on in-context learning with 4 exemplars (4-shot) and 10,000 queries. 
We compare with a zero-shot prediction that provides inherently privacy guarantee ($\eps=0$) and non-private ($\eps=\infty$) 4-shot prediction. Since the performance of in-context learning has a large variance \citep{pmlr-v139-zhao21c, min-etal-2022-rethinking}, we also compare our results with the performance of output aggregation ($\eps=\infty$(Agg)).  We set the number of exemplar-query pairs to 10 after subsampling and selected $\eps=\{1, 3, 8\}$ to achieve different levels of privacy.

\textbf{DP-ICL achieves a comparable performance with non-private ICL across all tasks (Table \ref{tab:main}).} Our findings indicate that the impact of considering privacy on accuracy is marginal.
For instance, the performance only drops by 0.04\% for SST-2 with $\eps=3$ on GPT-3 Babbage. 
% Surprisingly, for $\eps=8$, the performance exceeds the aggregated outputs for SST-2 and Amazon. 
Even for a conservative privacy budget of $\eps=1$, we observe that DP-ICL can significantly outperform the zero-shot prediction (e.g., over 20 \% for AGNews) depending on the dataset. 

\textbf{DP-ICL can be further improved via deploying advanced LLMs. }
By comparing the performance of GPT-3 Davinci and GPT-3 Babbage on SST-2, we find that the larger model leads to better performance across all $\eps$ for DP-ICL.
Take $\eps=1$ as an example; GPT-3 Davinci outperforms GPT-3 Babbage by $\sim$3.1\%.
In addition, we note that all our results can be further improved by simply replacing the GPT-3 API call with even more advanced LLMs, such as GPT-4.

% It is also possible that zero-shot learning may outperform both DP-ICL and non-private ICL under conditions where the in-context exemplars exhibit strong bias.
% In such cases, modifications, such as identifying more representative in-context exemplars, should be considered. We identify these considerations as directions for future research.

% \begin{table}[H]
%     \caption{ Results of DP-ICL for multiple privacy budgets}
%          % \vspace{-4mm}
%     \centering
%     % \resizebox{0.55\textwidth}{!}{%
%     \begin{tabular}{@{}cccccc@{}}
%     \toprule
%     \textbf{Dataset} & $\eps=0$ (0-shot)   & $\eps=1$ & $\eps=3$ & $\eps=8$ & $\eps=\infty$\\ \midrule
%    SST-2     & 91.00 & 95.06 &   95.80 &   95.92 & 96.05 \\ \midrule
%    Amazon    & 91.00 & 93.83 &   94.18 &   94.25 & 94.26 \\ \midrule
%    AGNews    & 46.00 & 73.16 &   78.98 &   79.77 & 81.43 \\ \midrule
%    TREC      & 20.00 & 25.83 &   26.75 &   27.18 & 28.42 \\ 
%     \bottomrule
%     \end{tabular}%
%     % }
%     \label{tab:main}
%     \end{table}

\begin{table}[H]
    \setlength{\abovecaptionskip}{0pt}
    \setlength{\belowcaptionskip}{0pt}
    \centering
    \caption{ \textbf{Results of DP-ICL for Text classification.} We compare our method with zero-shot prediction ($\eps=0$), four-shot predictions ($\eps=\infty$), and an aggregation of 10 four-shot predictions ($\eps=\infty$ (Agg)). For $\eps=\{1, 3, 8\}$, our DP-ICL generally surpasses zero-shot predictions and yields competitive performance relative to non-private predictions. }
    \resizebox{0.9\textwidth}{!}{%
    \begin{tabular}{@{}cccccccc@{}}
    \toprule
    \textbf{Dataset} & \textbf{Model} & $\eps=0$ (0-shot)   & $\eps=1$ & $\eps=3$ & $\eps=8$ & $\eps=\infty$ (Agg)& $\eps=\infty $\\ \midrule
   \multirow{2}{*}{SST-2} & Babbage & 86.58 & 91.97$_{0.49}$ &   92.83$_{0.28}$ &   92.90$_{0.24}$ & 92.87$_{0.09}$ & 91.89$_{1.23}$ \\ 
                          & Davinci & 94.15 & 95.11$_{0.35}$ &   95.80$_{0.21}$ &   95.83$_{0.21}$ & 95.73$_{0.13}$ & 95.49$_{0.37}$ \\ \midrule
   Amazon                 & Babbage & 93.80 & 93.83$_{0.33}$ &   94.10$_{0.22}$ &   94.12$_{0.20}$ & 94.10$_{0.11}$ & 93.58$_{0.64}$ \\ \midrule
   AGNews                 & Babbage & 52.60 & 75.49$_{1.46}$ &   81.00$_{1.14}$ &   81.86$_{1.22}$ & 82.22$_{2.16}$ & 68.77$_{11.31}$ \\ \midrule
   % \multirow{2}{*}{TREC} 
   TREC                   & Babbage & 23.00 & 24.48$_{3.58}$ &   26.36$_{5.19}$ &   26.26$_{5.61}$ & 26.32$_{5.33}$ & 27.00$_{7.72}$ \\ 
                          % & Davinci & 79.60 & 73.31$_{1.54}$ &   83.86$_{0.83}$ &   84.53$_{0.63}$ & 85.92$_{0.27}$ & 79.08$_{1.19}$ \\ 
    \bottomrule
    \end{tabular}%
    }
    \label{tab:main}
    \end{table}

% \begin{table}[H]
%     \caption{ Results of DP-ICL for multiple privacy budgets}
%          \vspace{4mm}
%     \centering
%     % \resizebox{0.55\textwidth}{!}{%
%     \begin{tabular}{@{}ccccccc@{}}
%     \toprule
%     \textbf{Dataset} & \textbf{Model} & $\eps=0$ (0-shot)  & $\eps=1$  & $\eps=3$ & $\eps=8$ & $\eps=\infty$  \\ \midrule
%    \multirow{2}{*}{SST-2} & Babbage & 86.58 & 91.97 &   92.83 &   92.90 & 92.87  \\ 
%                           & Davinci & 94.15 & 94.86 &   95.45 &   95.45 & 95.41  \\ \midrule
%    Amazon                 & Babbage & 93.80 & 93.83 &   94.10 &   94.12 & 94.10  \\ \midrule
%    AGNews                 & Babbage & 52.60 & 75.49 &   81.00 &   81.86 & 82.22  \\ \midrule
%    \multirow{2}{*}{TREC}  & Babbage & 23.00 & 24.48 &   26.36 &   26.26 & 26.32  \\ 
%                           & Davinci & 79.60 & 73.18 &   82.74 &   83.12 & 84.33  \\ 
%     \bottomrule
%     \end{tabular}%
%     % }
%     \label{tab:main}
%     \end{table}

\vspace{-2mm}
\subsection{DP-ICL for Document  Questions Answering}
\label{subsec:exp_dqa}

\vspace{-2mm}

Then, we consider the document questions answering task, which aims to answer questions via reasoning a given document. We adopt a dataset that originates from a Privacy Preserving Federated Learning Document VQA (PFL-DocVQA) competition. We directly leverage the token extracted from the OCR model as the given context and use LLMs to generate answers to questions.
Here, we use the open-source model OpenLLaMA-13B \citet{openlm2023openllama} and 1-shot ICL as a cost-effective choice to conduct experiments, and our methods are readily generalizable to other LLMs.
We employ three metrics, ROUGE-1, BLEU, and normalized Levenshitein similarity, to comprehensively evaluate our proposed methods. Higher values in these metrics indicate better performance. See \cref{appendsub:setupla} for more details and examples. 

\vspace{-1mm}
For baseline methods, we again include evaluations for zero-shot prediction ($\eps=0$), 1-shot prediction ($\eps=\infty$), and non-private aggregation ($\eps=\infty$(Agg)) where we perform aggregation without introducing noise.
We compare embedding space aggregation and keyword space aggregation by PTR approaches.\footnote{Here, we did not implement the joint EM method, given that the output domain for question answering could potentially be infinite.}  The ensemble, query, and output candidate sizes are all set to 100.

\begin{table}[H]
\vspace{-2mm}
    \caption{ \textbf{Results of DP-ICL Applied to Document Question Answering.} We use three baselines including zero-shot predictions ($\eps=0$), 1-shot ICL ($\eps=\infty$), as well as non-private aggergation ($\eps=\infty$(Agg)).  In the non-private aggregation setting, we use either embedding or keyword methods without adding privacy noise. }
         % \vspace{-4mm}
    \centering
    \resizebox{0.9\textwidth}{!}{%
    \begin{tabular}{@{}cccccccc@{}}
    \toprule
    \textbf{Methods} & \textbf{Metrics} & $\eps=0$ (0-shot)   & $\eps=1$ & $\eps=3$ & $\eps=8$ & $\eps=\infty$ (Agg)& $\eps=\infty$ (1-shot)\\ \midrule
   \multirow{3}{*}{Embedding} 
   & ROUGE-1  $\uparrow$   & 19.05 & 37.78$_{0.35}$ &   37.91$_{0.19}$ &   38.06$_{0.15}$ & 37.97 & 50.68 \\ 
   & BLEU  $\uparrow$      &  4.42 &  6.49$_{0.20}$ &    6.51$_{0.04}$ &    6.54$_{0.16}$ & 6.43  & 24.03 \\
   & Levenshtein $\uparrow$ & 16.15 & 30.39$_{0.50}$ &   30.71$_{0.45}$ &   30.88$_{0.06}$ & 30.94 & 49.30 \\ 
  \midrule
      \multirow{3}{*}{   \begin{tabular}[c]{@{}c@{}}Keyword\\ by PTR\end{tabular}}    
   & ROUGE-1 $\uparrow$    & 19.05 & 59.92$_{0.60}$ &   60.40$_{0.50}$ &   60.66$_{0.61}$ & 62.42 & 50.68 \\ 
   & BLEU       $\uparrow$ & 4.42  & 23.32$_{0.51}$ &   23.67$_{0.45}$ &   23.93$_{0.45}$ & 25.10 & 24.03 \\
   & Levenshtein $\uparrow$& 16.15 & 51.47$_{0.67}$ &   52.05$_{1.06}$ &   52.47$_{1.09}$ & 52.42 & 49.30 \\ 
  % \midrule
  %   \multirow{3}{*}{   \begin{tabular}[c]{@{}c@{}}keyword\\ with ranking\end{tabular}}    
  %  & ROUGE-1     & 19.05 & tbd &   tbd &   tbd &  & 50.68 \\ 
  %  & BLEU        & 4.42  & tbd &   tbd &   tbd &  & 24.03 \\
  %  & Levenshtein & 16.15 & tbd &   tbd &   tbd &  & 49.30 \\ 
    \bottomrule
    \end{tabular}%
    }
    \label{tab:gen_ocr_llama}
    \end{table}

% switch the order, highlight the limitation

\textbf{DP-ICL achieves competitive results to non-private aggregations even with $\eps$=1 (Table \ref{tab:gen_ocr_llama}).} 
 Remarkably, our empirical results suggest that adopting differential privacy does not lead to substantial performance degradation. For instance, the decline in ROUGE-1 scores when shifting from $\eps = \infty $(Agg) to $\eps = 1$ is less than 3\% for keyword space aggregation (KSA) by PTR methods. For embedding space aggregation (ESA), the decrease is even more minimal at 0.19\%.
 
Another noteworthy point is that the KSA method significantly surpasses the ESA approach, even exceeding the results of standard 1-shot ICL. This is because the consensus of keywords in output sentences leads to a more reliable answer. 
This performance drop in the ESA is mainly due to two factors: 
(1) information loss during projecting outputs into an embedding space, and (2) the lack of high-quality candidates generated by zero-shot predictions of OpenLLaMA-13B models.
We think employing advanced LLMs and embedding reconstruction methods could mitigate these drawbacks.

\vspace{-2mm}

\subsection{DP-ICL for Dialog Summarization }
\label{subsec:exp_ds}
\vspace{-2mm}

% We use the SAMSum dialog summarization data \cite{gliwa-etal-2019-samsum} to conduct experiments following the previous work on differentially private language generation \citep{he2023exploring}. \footnote{\citet{he2023exploring} use the internal version of GPT-3 to conduct experiments where we cannot obtain the results using the same model.} This task is much more challenging than the previously evaluated one since the output can be multiple long sentences. We evaluate all three proposed methods: embedding space aggregation (ESA), keywords by PTR, and keywords by jointEM, using 4-shot ICL and GPT-3 Davinci API following \citep{he2023exploring}'s setup. 

We evaluate on the SAMSum dialog summarization dataset \cite{gliwa-etal-2019-samsum}. This task is much more challenging than previous tasks because the output can be multiple long sentences.
We consider all three proposed methods: embedding space aggregation (ESA), keyword by PTR, and keyword by jointEM, using 4-shot ICL and GPT-3 Davinci API.
For the keyword space aggregation (KSA), GPT-3 is again used to reconstruct the answers with extracted keywords within prompts. We compare three baselines: zero-shot learning, 4-shot ICL, and predictions of non-private aggregation.  More details of the evaluation are in  \cref{appendsub:setupla}.

\begin{table}[H]
\vspace{-2mm}

    \caption{ \textbf{Results of DP-ICL for dialog summarization. } We again compare with zero-shot predictions ($\eps=0$), 4-shot ICL ($\eps=\infty$), as well as non-private aggregation ($\eps=\infty$(Agg)). We report three variants of our private aggregation approaches with $\eps=\{1,3,8\}$.  }
    \centering
    \resizebox{0.9\textwidth}{!}{
    \begin{tabular}{@{}cccccccc@{}}
    \toprule
    \textbf{Method} & \textbf{Metrics} & $\eps=0$ (0-shot)   & $\eps=1$ & $\eps=3$ & $\eps=8$ & $\eps=\infty$ (Agg)& $\eps=\infty$ (4-shot)\\ \midrule
   \multirow{3}{*}{Embedding}   
    & ROUGE-1 $\uparrow$    & 35.31 & 38.21$_{0.39}$ & 38.92$_{0.24}$ & 39.62$_{0.40}$ & 40.27 & 43.32 \\ 
    & ROUGE-2 $\uparrow$    & 12.65 & 14.55$_{0.66}$ & 15.18$_{0.43}$ & 15.43$_{0.46}$ & 16.52 & 19.08 \\
    & ROUGE-L $\uparrow$    & 27.02 & 29.85$_{0.61}$ & 30.86$_{0.22}$ & 31.24$_{0.45}$ & 32.29 & 34.78 \\
    \midrule
     \multirow{3}{*}{   \begin{tabular}[c]{@{}c@{}}Keyword\\ by joint EM\end{tabular}}    
    & ROUGE-1  $\uparrow$  & 35.31 & 40.02$_{0.37}$ & 40.98$_{0.47}$ & 41.21$_{0.58}$ & 42.40 & 43.32 \\ 
    & ROUGE-2 $\uparrow$   & 12.65 & 15.67$_{0.60}$ & 16.49$_{0.79}$ & 16.31$_{0.43}$ & 15.61 & 19.08 \\ 
    & ROUGE-L $\uparrow$   & 27.02 & 30.46$_{0.73}$ & 31.76$_{0.26}$ & 31.84$_{0.34}$ & 32.60 & 34.78 \\ 
    \midrule
      \multirow{3}{*}{   \begin{tabular}[c]{@{}c@{}}Keyword\\ by PTR\end{tabular}}    
    & ROUGE-1 $\uparrow$   & 35.31 & 38.54$_{0.47}$ & 39.09$_{0.39}$ & 39.71$_{0.21}$ & 41.03 & 43.32 \\
    & ROUGE-2 $\uparrow$   & 12.65 & 14.42$_{0.54}$ & 14.32$_{0.45}$ & 14.60$_{0.39}$ & 15.91 & 19.08 \\
    & ROUGE-L $\uparrow$   & 27.02 & 29.58$_{0.45}$ & 30.18$_{0.44}$ & 30.56$_{0.30}$ & 32.47 & 34.78 \\
    \bottomrule
    \end{tabular}%
    }
    \label{tab:gen_main}
    \end{table}

\textbf{Employing DP-ICL offers consistent advantages over zero-shot learning across all methods (Table \ref{tab:gen_main}).} 
Notably, under privacy constraints with $\eps = 1$, our most effective approach, keyword by joint EM, yielded an improvement of approximately 4.5\% in ROUGE-1, 3.0\% in ROUGE-2, and 3.4\% in ROUGE-L than zero-shot learning.
This result is only marginally lower, by $\sim$1\% on average than the performance achieved with non-private aggregations. 
% In the previous method DP-LoRA \citep{he2023exploring}, the average degradation with $\eps = 1$ is greater than 4.  
Interestingly, the keyword by joint EM outperforms the keyword by PTR; this advantage is primarily due to joint EM also releasing the order of frequency.
Another finding is that our non-private aggregation methods performed worse than in-context learning predictions, even for keyword space aggregation methods. 
We leave the question of optimizing the utilization of extracted keywords as future research. 
% These findings indicate avenues for future research, particularly in optimizing the utilization of extracted keywords.

\vspace{-2mm}

\subsection{Ablation Studies}
\label{subsec:exp_abl}

\vspace{-2mm}

 We also conduct an ablation study on the dialog summarization task, as it is the most challenging task, with varying number of queries and number of ensembles. 
 Experimental results on text classification task are provided in \cref{append:addeva_class}, and more findings related to language generation tasks are present in  \cref{append_gen}.  Here, we set the differential privacy parameter $\eps=3$. 

  \begin{figure}[H]
  \centering
  \begin{tabular}{cccc}
    \includegraphics[width=0.22\textwidth]{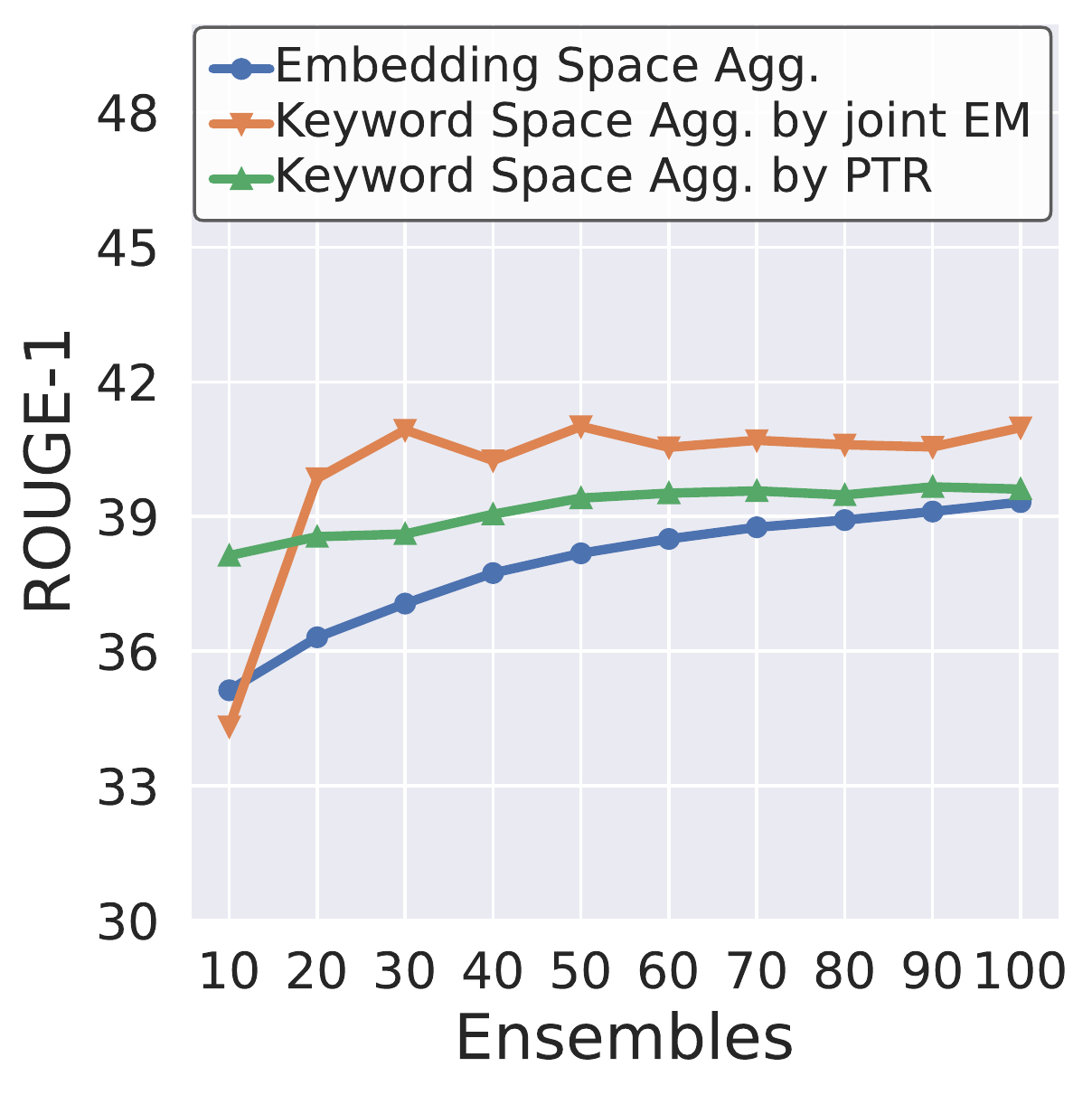} & 
   \includegraphics[width=0.22\textwidth]{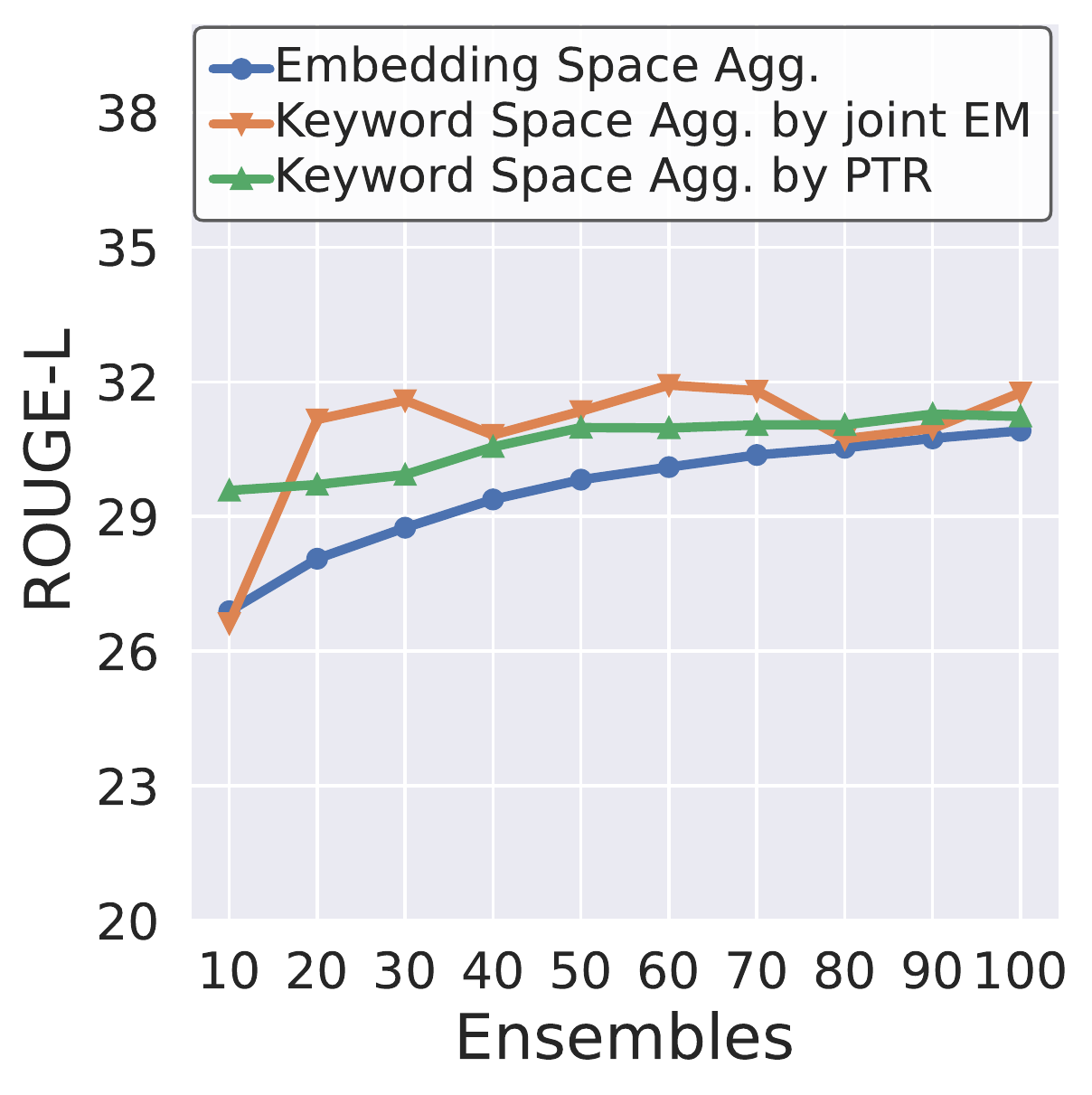}&
    \includegraphics[width=0.22\textwidth]{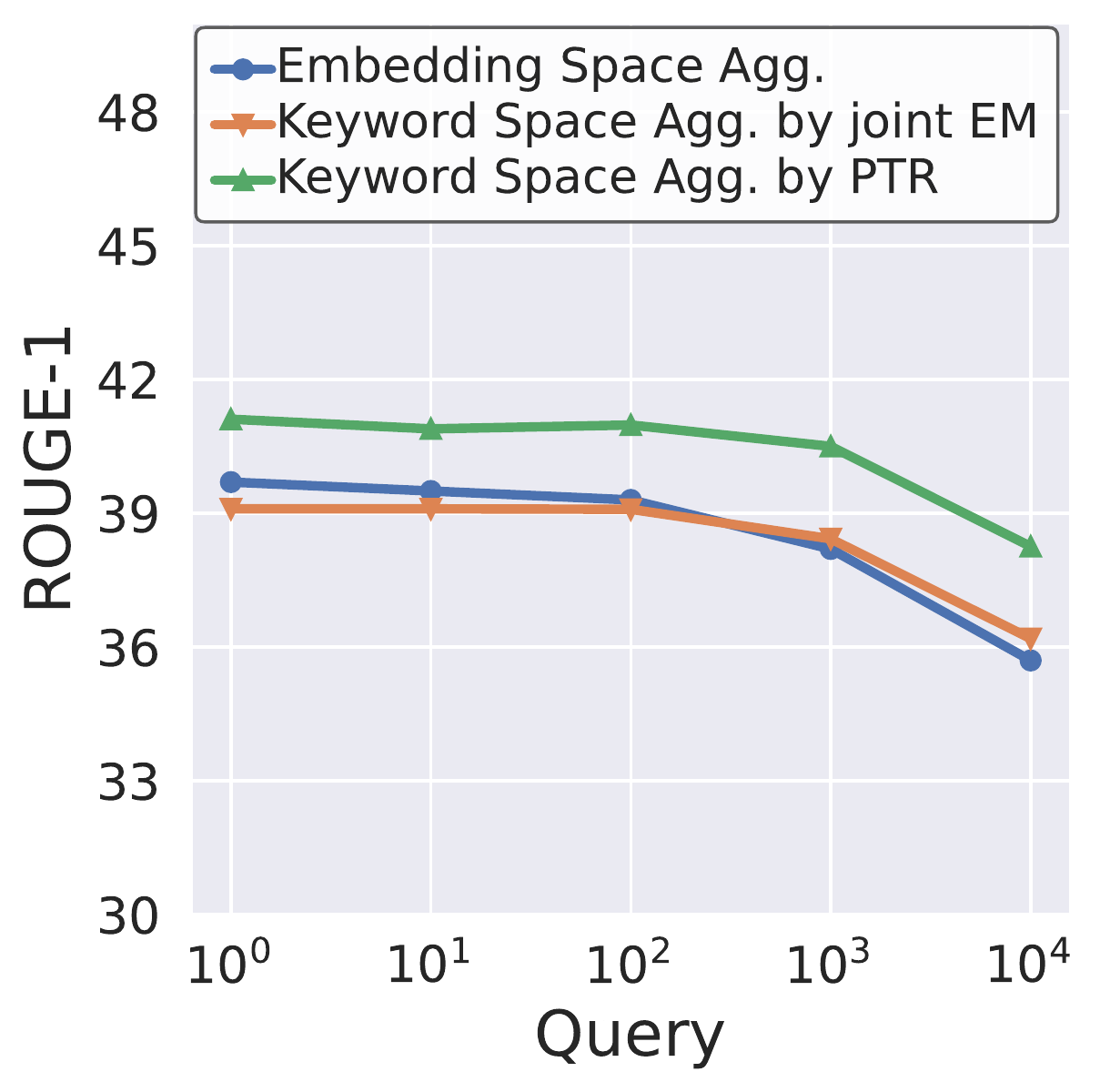} & 
   \includegraphics[width=0.22\textwidth]{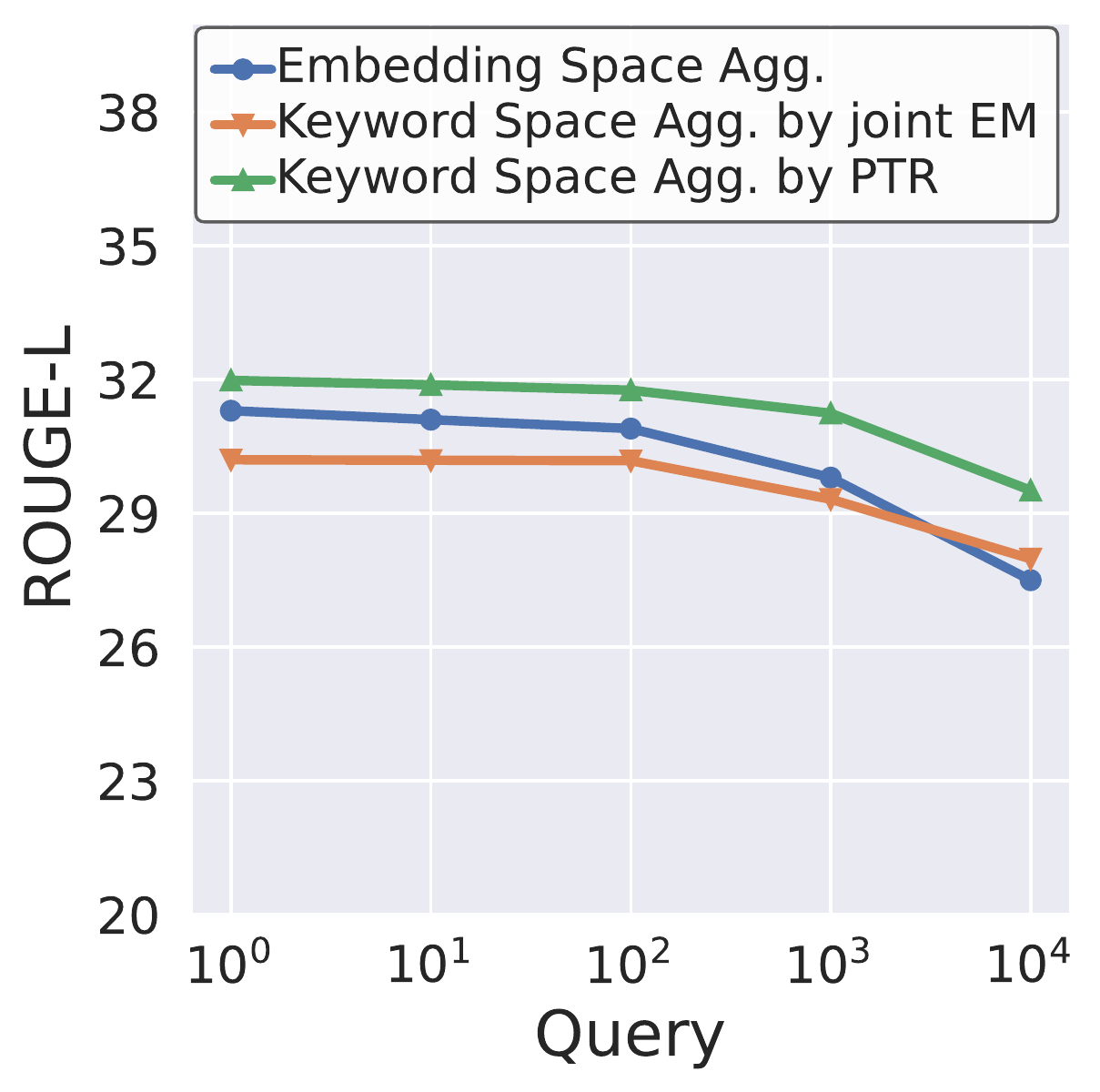} \\
     \multicolumn{2}{c}{(a) Varying numbers of ensembles }  & \multicolumn{2}{c}{(b) Varying numbers of queries}   \\
  \end{tabular}
  \vspace{-3mm}
      \caption{Ablation studies on dialog summarization task. Figure \ref{Fig:abl_query_full} \& \ref{Fig:abl_ensemble_full} present full results. }
      \label{Fig:abl_query_ensemble}
  \end{figure}
  
\textbf{Effectiveness across numbers of ensembles (\cref{Fig:abl_query_ensemble}(a)).}  Our prior evaluation discussed in \cref{subsec:exp_ds} utilizes an ensemble of 100 teachers for all methods.  In this section, we vary the ensemble size from 10 to 100. 
It is noteworthy that increasing the ensemble size results in raised subsampling rates, thereby introducing additional noise into the aggregation process. 
At the same time, a larger ensemble size could also generate a more reliable consensus. 
We observe a clear trend of performance improvement when increasing ensembles for embedding space aggregation and keyword by PTR. 
However, the result of the keyword by joint EM approach shows more fluctuations and reaches a high performance with 30 ensembles. 

\textbf{Effectiveness across numbers of queries (\cref{Fig:abl_query_ensemble}(b)).} We then investigate the impact of varying the number of queries across the set $\{1,10,100,1000,10000\}$. It is important to note that increasing the number of queries inherently strengthens privacy protection, causing a higher noise level. Our results indicate that performance degradation in KSA by joint EM remains marginal up to $10^3$ queries, suggesting its suitability for real-world applications.

\section{Related Works}

% \textbf{In-context learning and prompt.} The scaling of 

% \textbf{Differential privacy in natural language processing.}

\textbf{Differentially Private Language Models.} The existing research on differentially private language models \citep{Li2021LargeLM, Yu2021DifferentiallyPF, Bu2022DifferentiallyPB, he2023exploring} primarily focused on improving DP-SGD \citep{abadi2016deep} for training language models. In this paradigm, noise is introduced to the gradient during the model's training to ensure privacy. 
However, as the scale of the large language models significantly increased, fine-tuning has become much more challenging, making this approach less practical. 
We provide detailed qualitative and quantitative comparisons between DP-ICL and DP-SGD in \cref{appendsec:privateprediction}.

\textbf{Concurrent works on Differentially Private In-Context Learning.} 
For differentially private in-context learning, concurrent to our work, \citet{Duan2023FlocksOS} propose an approach that privately labels a publicly available dataset and then uses the newly labeled data pairs as demonstrations, while our approach does not rely on public data.   
Later, \citet{Tang2023PrivacyPreservingIL} present an approach that privately generates in-context exemplars directly via prompting and achieves effective ICL.
It is essential to underscore that both approaches are \textbf{restricted to tasks involving limited label spaces}, such as text classification and word extraction. 
By contrast, we show that DP-ICL can obtain competitive performance on SAMSum and DocVQA, which are considerably more complex and challenging tasks in language generation. Our methodology and compelling results indicate that our methods can be broadly applied across various natural language processing tasks.

% Another direction of generating private texts is through metric differential privacy \citep{Feyisetan2019PrivacyAU}

\section{Discussion and Future Works}

In this paper, we initiate the study of incorporating in-context learning with differential privacy. We developed a unified framework for privatizing ICL based on the famous ``sample-and-aggregate'' paradigm, and we propose several instantiations for the private aggregation for the task of text classification and language generation. 

% In this study, we present a new approach for differentially private in-context learning, which relies on the insight of ensembling.
% For basic text classification tasks, we use the Report-Noisy-Max mechanism and achieve results that are comparable to non-private in-context learning.
% Moreover, we extend our approach to tackle more complicated and general language generation tasks via mapping the sentences to either sentence embedding space or keyword space. 

While our method exhibits strong performance, there are multiple directions for future research. For instance, for the Embedding Space Aggregation method, a better-trained embedding-to-text model may yield further improvements in the model performance. 
Additionally, DP-ICL relies on dividing the exemplars into disjoint subsets and queries the LLM with each of the subsets. 
As the number of subsets increases, the computational efficiency, while still significantly better compared with directly fine-tuning the LLM, will be larger. 
The efficiency-utility tradeoff for the choice of the number of subsets in DP-ICL is an interesting problem for future works. 

% While more such subsets may improve the final model performance, the computational efficiency will also be larger. Determining 

% a limitation of DP-ICL lies in the necessity to generate multiple outputs to ensure reliable performance after adding noise. Future work could investigate strategies for determining the optimal ensemble size before making predictions, thereby conserving computational resources.

% Compared with prior work in private learning via DP fine-tuning, DP-ICL offers an improved privacy-utility-computation tradeoff with additional flexibility in model compatibility and data editing.

% However, DP-ICL cannot answer an infinite number of queries from an attacker: in this work, we consider the threat model of non-colluding adversaries who can each adaptively ask up to $k=100,000$ queries.
% If two adversaries collude together, then the privacy loss will be equivalent to the case where $k=200,000$.

% Future work can extend DP-ICL to text generation tasks and employ more advanced LLMs to allow for an even stronger attacker and enhance utility.

% \section*{Acknowledgments}
% This work was supported in part by the National Science Foundation under grants CNS-2131938, CNS-1553437, CNS-1704105, the ARL’s Army Artificial Intelligence Innovation Institute (A2I2), the Office of Naval Research Young Investigator Award, the Army Research Office Young Investigator Prize, Schmidt DataX award, and Princeton E-ffiliates Award, and Princeton's Gordon Y. S. Wu Fellowship.

\newpage

\bibliography{ref}
\bibliographystyle{iclr2024_conference}

\newpage
\appendix

\newpage
\section{Details and Pseudocode of DP In-context learning in Section \ref{sec:dpicl}}
\label{append_algo}

In this appendix, we provide full details of our DP-ICL algorithms. These algorithms include text classification through RNM-Gaussian (\cref{appendsub:alg_rnm}), language generation via embedding space aggregation (\cref{appendsub:alg_esa}), and language generation via keyword space aggregation (\cref{appendsub:alg_ksa}).

\subsection{text classification via RNM-Gaussian}
\label{appendsub:alg_rnm}

Our method is detailed in \cref{alg:dpquery}, and further privacy analysis is detailed in \cref{appendsub:analysis_rnm}.

  \begin{algorithm}[H]
    \caption{Differentially Private In-Context Learning  via RNM-Gaussian}
    \label{alg:dpquery}
    \begin{algorithmic}[1]
      \REQUIRE Private data $D$, query $Q$, model $\LLM$, noise $\sigma$, number of subsets $N$  
      \STATE \textbf{Partition} $D_1, D_2, \dots, D_N \leftarrow D$. 
      \FOR{$i \in \{1, \dots, N\}$}
        \STATE Form exemplar-query pair $D_i^Q = D_i \cup \{Q\}$.
        \STATE Obtain model output $o_i(Q) = \LLM(D_i^Q)$.
        \STATE Convert $o_i(Q)$ to a one-hot vector with length equal to the number of classes. 
      \ENDFOR
      \STATE Sum the one-hot vectors into a histogram $\textbf{H}$. 
      \STATE Add noise to $\N\left(0, \sigma^2\right)$ to each entry of $\textbf{H}$. 
      \STATE Report the top-1 bin from $\textbf{H}$.
    \end{algorithmic}
  \end{algorithm}

\subsection{Language Generation via Embedding Space Aggregation (ESA)}
\label{appendsub:alg_esa}

In \cref{alg:dpemb}, we present the full descriptions of our ESA method. The main idea is to project those output sentences into embedding space, get a differentially private mean, and map it back to sentence space. Further privacy analysis, including how to compute $\sigma$, is presented in \cref{appendix:priv-ESA}.

\newcommand{\zeroshot}{O_C}

\begin{algorithm}
\caption{Differentially Private In-Context Learning via Embedding}
\label{alg:dpemb}
\begin{algorithmic}[1]
\REQUIRE Private data $D$, query $Q$, model $\LLM$, noise $\sigma$, number of subsets $N$, 
% subsampling $q$, 
public candidate sentences obtained by zero-shot predictions $\zeroshot$ with total number of $C$
% \STATE \textbf{Subsample} $q \%$ of the data.
\STATE \textbf{Partition} $D_1, D_2, \dots, D_N \leftarrow D$. 
% \label{step:embed-start}
\FOR{$i \in \{1, \dots, N\}$}
    \STATE Form exemplar-query pair $D_i^Q = D_i \cup \{Q\}$.
    \STATE Obtain model output sentence $O_i(Q) = \LLM(D_i^Q)$.
    \STATE Project $O_i(Q)$ into embedding vector $E_i(Q)$. 
\ENDFOR
\STATE Take the mean of all embedding vectors $E = \frac{1}{N} \sum_{i=1}^N E_i$.
\STATE Adding noise and obtain the noisy embedding  $\widehat E = E + \N\left(0, \sigma^2 I \right)$. 
% \STATE Obtain the average embedding  $\hat{E}$ of $\E$
\STATE \textbf{Return} the sentence $\argmax_{o \in \zeroshot} \textbf{CosineSimilarity}(o, \widehat E)$.
% \alglinelabel{step:embed-last}
\end{algorithmic}
\end{algorithm}

\subsection{Language Generation via Keyword Space Aggregation (KSA)}
\label{appendsub:alg_ksa}

The keyword space aggregation (KSA) algorithm is demonstrated in \cref{alg:ksa}. 
We illustrate two differential private approaches for selecting keywords in the following subsections, including joint Exponential Mechanism (\cref{appendsub:alg_ksa_jem}) and Propose-Test-Release (\cref{appendsub:alg_ksa_ptr}).

\begin{algorithm}
\caption{Keyword Space Aggregation}
\label{alg:ksa}
\begin{algorithmic}[1]
\REQUIRE Private data $D$, query $Q$, model $\LLM$, noise $\sigma$, number of subsets $N$, 
% subsampling $q$,  
public candidates obtained by zero-shot predictions $O_C$ with total number of $C$, maximum token length $M$, $\text{method} \in \{ \text{JEM}, \text{PTR} \}$. 
% \STATE \textbf{Subsample} $q \%$ of the data.  
\STATE \textbf{Partition} $D_1, D_2, \dots, D_N \leftarrow D$.  
\FOR{$i \in \{1, \dots, N\}$}
    \STATE Form exemplar-query pair $D_i^Q = D_i \cup \{Q\}$.
    \STATE Obtain model output sentence $O_i(Q) = \LLM(D_i^Q)$.
    % \STATE 
    % Divide the $O_i(Q)$ into words $\{o^1_i, o^2_i, \dots, o^M_i\}$ and filter the meaningless words. 
\ENDFOR

\STATE For each token, count the number of sentences in $\{O_i(Q)\}$ it appears, and form a histogram $\bH$. 

\IF{method = JEM}
\STATE \textbf{Return} $\texttt{JointEM}(k, \bH)$. (\cref{alg:joint})
\ELSIF{method = PTR}
\STATE $\widehat k = \texttt{FindBestK}(\bH)$.  (\cref{alg:rnm-find-k})
\STATE \textbf{Return} \texttt{TopKwithPTR}($\widehat k, \texttt{FindBestK}(\bH)$).  (\cref{alg:ptr})
\ENDIF
% \STATE Add noise to $\N\left(0, \sigma^2\right)$ to each class of $\textbf{H}$. \tong{more techniques to present}
% \STATE Report the maximum one 
\end{algorithmic}
\end{algorithm}

\subsubsection{KSA via Joint Exponential Mechanism.} 
\label{appendsub:alg_ksa_jem}

\begin{algorithm}
\caption{\texttt{JointEM} \citep{gillenwater2022joint}}
\label{alg:joint}
\begin{algorithmic}[1]
\REQUIRE Vector of item counts \(c_1, \ldots, c_d\), number of items to estimate \(k\), privacy parameter \(\varepsilon\)
\STATE Sort and relabel items so \(c_1 \geq c_2 \geq \cdots \geq c_d\)
\STATE Construct matrix \(\tilde{U}\) by \(\tilde{U}_{ij} = -(c_i - c_j) - \frac{d(k-i) + j}{2dk}\)
\STATE Sort \(\tilde{U}\) in decreasing order to get \(\tilde{U}_{(1)}, \ldots, \tilde{U}_{(dk)}\), storing the (row, column) of each \(\tilde{U}_{(a)}\) as \( (r(a), c(a)) \)
\STATE Initialize \( n_1, \ldots, n_k \leftarrow 0 \)
\STATE Initialize set of non-zero \( n_i, N \leftarrow \emptyset \)
\STATE Initialize \( b \leftarrow 0 \)
\FOR{\( a = 1, \ldots, dk \)}
    \STATE \( n_{r(a)} \leftarrow c(a) - (r(a) - 1) \)
    \STATE \( N \leftarrow N \cup \{ r(a) \} \)
    \IF{\( |N| = k \)}
        \STATE break
    \ENDIF
    \STATE Set \( \tilde{m}(\tilde{U}_{(a)}) \leftarrow 0 \), and set \( b \leftarrow a \)
\ENDFOR
\STATE Set \( p \leftarrow \prod_{r \in [k]} n_r \)
\STATE Compute \( \tilde{m}(\tilde{U}_{(b+1)}) \leftarrow p / n_{r(a)} \)
\FOR{\( a = b+2, \ldots, dk \)}
    \STATE Set \( p \leftarrow p / n_{r(a)} \)
    \STATE Compute \( \tilde{m}(\tilde{U}_{(a)}) \leftarrow p \)
    \STATE Update \( n_{r(a)} \leftarrow n_{r(a)} + 1 \)
    \STATE Update \( p \leftarrow p \cdot n_{r(a)} \)
\ENDFOR
\STATE Sample a utility \( \tilde{U}_{ij} \) from: $\mathbb{P}\left[\tilde{U}_{i j}\right] \propto \tilde{m}\left(\tilde{U}_{i j}\right) \exp \left(\frac{\varepsilon\left\lceil\tilde{U}_{i j}\right\rceil}{2}\right)$
\STATE Initialize size-\( k \) output vector \( s \) with \( s_i \leftarrow j \)
\FOR{\( i' = 1,2, \ldots, i-1, i+1, \ldots, k \)}
    \STATE Compute \( t_{i'}(\tilde{U}_{ij}) \) by iterating through row \( i' \) of \( \tilde{U} \)
    \STATE Sample \( s_{i'} \) uniformly from \( t_{i'}(\tilde{U}_{ij}) \backslash \{ j, s_1, s_2, \ldots, s_{i'-1} \} \)
\ENDFOR
\STATE Return Vector of item indices \( s \)
\end{algorithmic}
\end{algorithm}

The main idea of the joint exponential mechanism~\citep{gillenwater2022joint} is to provide a mechanism that samples \textit{sequences} of items rather than use variants of the exponential mechanism that may require composition over the number of tokens~\citep{durfee2019practical}.
We create the ``public domain'' for our summarization tasks by creating a histogram of counts for the words in the dialogue we want to summarize. We increase these counts for each exemplar. Note that just creating a histogram over the outputs of all exemplars would violate privacy. This is a challenge in extending KSA-JOINT to the infinite domain. We use this histogram to initialize a data structure to efficiently sample with the joint exponential mechanism. Our implementation uses the code from~\citep{gillenwater2022joint}.

\subsubsection{KSA via Propose-Test-Release (KSA-PTR)}
\label{appendsub:alg_ksa_ptr}

\textbf{Notations.} 
We use $N$ to denote the total number of tokens (e.g., $N = 50,000$). We use $\bH$ to denote the histogram for the counts of each token, and we use $\bH_{(j)}$ to denote the $j$th highest count, i.e., $\bH_{(1)} \ge \bH_{(2)} \ge \ldots \ge \bH_{(N)}$. 

The main idea of KSA-PTR is that, for the task of releasing the top-$k$ index set of a voting histogram, if $\bH_{(k)} - \bH_{(k+1)} > 2$, then the top-$k$ indices are exactly the same for all the neighboring datasets. Hence, one can release the exact top-$k$ indices without any randomness. However, we need to test whether $\bH_{(k)} - \bH_{(k+1)} > 2$ in a differentially private way, where we can leverage the famous propose-test-release paradigm \citep{dwork2009differential}, as shown in \cref{alg:ptr}. 

\begin{algorithm}[H]
\caption{\texttt{TopKwithPTR}}
\label{alg:ptr}
    \begin{algorithmic}[1]
      \REQUIRE $k$ -- the number of top counted tokens to release; $\bH$ -- histogram for the counts of each token; $\delta$ -- failure probability
      \STATE \textbf{Set} $d_k := \bH_{(k)} - \bH_{(k+1)}$. 
      \STATE \textbf{Set} $\widehat d_k := \max(2, d_k) + \N(0, 4\sigma^2) - \Phi(1-\delta; 0, 2\sigma)$. 
%  \IF{ $\widehat d_k > 2$} \STATE \textbf{Return} the exact top-$k$ tokens.
% \ENDIF
    \STATE \textbf{If} $\widehat d_k > 2$, \textbf{Return} the exact top-$k$ tokens.
    \STATE \textbf{Else} Terminate (or use zero-shot learning).
    \end{algorithmic}
\end{algorithm}

As we can see, such an algorithm can have the highest utility when we choose the $k$ that maximizes $\bH_{(k)} - \bH_{(k+1)}$. 
Hence, to further improve the utility of the algorithm, we can select $k$ in a data-dependent way, i.e., we release $\argmax_k \bH_{(k)} - \bH_{(k+1)}$ in a differentially private way (which is another Report-Noisy-Max) using Exponential mechanism. 

\begin{algorithm}[H]
    \caption{\texttt{FindBestK}}
    \label{alg:rnm-find-k}
    \begin{algorithmic}[1]
      \REQUIRE $\bH$ -- histogram for the counts of each token
      \STATE Compute histogram gap $d_k := \bH_{(k)} - \bH_{(k+1)}$ for each $k = 1 \ldots N-1$. 
      \STATE \textbf{Return} $\argmax_k \{ d_k + r(k) + \gumbel(4 /\eps) \}$ 
    \end{algorithmic}
\end{algorithm}

Here, $r(k)$ is a regularizer independent of the dataset, e.g., we can set $r(k) = - \infty$ for any $k > 30$ and $k < 15$, if we don't want to return more than 30 or less than 15 tokens.

% \begin{algorithm}[h]
% \SetAlgoLined
% \SetKwInOut{Input}{input}
% \SetKwInOut{Output}{output}
% \Input{
% $k$ -- the number of top counted tokens to release; 
% $\bH$ -- histogram for the counts of each token;
% $\delta$ -- failure probability
% }

% Set $d_k := \bH_{(k)} - \bH_{(k+1)}$. 

% Set $\widehat d_k := \max(2, d_k) + \N(0, 4\sigma^2) - \Phi(1-\delta; 0, 2\sigma)$. 

% \If{$\widehat d_k > 2$}{
%     \Return{the exact top-$k$ tokens.}
% }\Else{
%     \Return{Terminate (or use zero-shot learning).}
% }

% \caption{Private Top-$k$ Selection with PTR}
% \label{alg:ptr}
% \end{algorithm}

\newpage
\section{Privacy Analysis}
\label{appendix_privacyanalysis}

In this section, we review the important properties of differential privacy and provide the privacy analysis for the algorithms introduced in Section \ref{sec:dpicl}. We first state the formal DP definition.

\begin{definition}[Differential Privacy \citep{dwork2006calibrating}] 
\label{def:DP}
For $\eps, \delta \ge 0$, a randomized algorithm $\M : \MS(\cX)\rightarrow \cY$ is 
{\em $(\eps, \delta)$-differentially private} if for every neighboring dataset pair $D, D' \in \MS(\cX)$, we have: 
\begin{align*}
\forall\ T\subseteq \cY\ \Pr[\M(D) \in T] \le e^\eps\cdot \Pr[\M(D') \in T] + \delta
\end{align*}
where the randomness is over the coin flips of $\M$. 
\end{definition}

\textbf{Post-processing Property.}
Differential privacy exhibits a robust post-processing property. Informally, this means that if a mechanism is differentially private, then any post-processing applied to the output of that mechanism is also differentially private. This property is crucial for enabling flexible analysis of privately released data. 
% In particular, although the task we consider is top-1 query answering, the mechanism we use allows us to release the entire confidence score vector.

\begin{lemma}[Post-processing \citep{dwork2006calibrating}]
\label{lemma:post-processing}
If $\M: \MS(\cX) \rightarrow \cY$ is $(\eps, \delta)$-differentially private and $f: \cY \rightarrow \cZ$ is an arbitrary (possibly randomized) function, then the composed mechanism $f \circ \M: \MS(\cX) \rightarrow \cZ$ is also $(\eps, \delta)$-differentially private.
\end{lemma}

\textbf{(Adaptive) Composition of Differential Privacy.} 
In practice, multiple differentially private mechanisms may be applied to the same dataset. Crucially, multiple DP mechanisms can be \emph{adaptively} composed in the sense that the output of one mechanism can be used as an input to another mechanism, denoted as $\M(D) = \M_1 \circ \M_2 (D) := (\M_1(D), \M_2(D, \M_1(D)))$. Differential privacy offers strong composition guarantees, that help quantify the cumulative privacy loss resulting from these combined mechanisms. These guarantees are provided by various composition theorems or privacy accounting techniques, including the basic composition theorem \citep{dwork2006our}, advanced composition theorem \citep{dwork2010boosting}, and Moments Accountant \citep{abadi2016deep}. For example, the basic composition theorem states that if $\M_1$ is $(\eps_1, \delta_1)$-DP and $\M_2$ is $(\eps_2, \delta_2)$-DP, then the adaptive composition of $\M_1$ and $\M_2$ is $(\eps_1 + \eps_2, \delta_1 + \delta_2)$-DP. 

Consider two attackers: the first asks their allotted $k$ queries in one batch and then observes the answers, the second asks each query sequentially and incorporates information gained from observing the answer to the current query into the next query. The second attacker is certainly stronger, and this increased strength is captured by adaptive composition.

% We note that under composition we can extend our threat model to consider an arbitrary number of colluding users. 
% We primarily consider the threat model of non-colluding adversaries who can each adaptively ask up to $k=10,000$ queries. 
% If two adversaries collude together, then the privacy loss will be equivalent to the case where $k=20,000$ \citep{vadhan2021concurrent, vadhan2022concurrent, lyu2022composition}:

% \begin{definition}[Privacy Guarantee for Colluding Adversaries]
%     Let the privacy guarantee for non-colluding adversaries limited to at most $k$ queries be given by $P(k)=\eps_k$, then the privacy guarantee for $m$ colluding adversaries each limited to at most $k$ queries is given by $P(mk)=\eps_{mk}$. 
% \end{definition}\label{eq:collude}

% For Gaussian mechanism where the sensitivity of underlying function is 1 and noise scale $\sigma$, the PRV is $Y = \log( P(o)/Q(o) ), o \sim P$ where $P = \N(1, \sigma^2), Q=\N(0, \sigma^2)$, and $P(\cdot), Q(\cdot)$ are the density functions of $P, Q$. 

% The privacy loss random variable for the Gaussian mechanism is given by $Y = \frac{\Delta f}{\sigma} Z$, where $Z$ follows a standard normal distribution, and for the Poisson subsampled Gaussian mechanism, it is given by $Y = \log(1 + q\cdot(\exp(\frac{\Delta f}{\sigma}Z)-1))$, where $q$ is the subsampling probability. This method of privacy cost accounting enables efficient and accurate tracking of privacy costs when using complex compositions of differentially private mechanisms.

\textbf{Privacy Amplification by Subsampling.} 
Privacy amplification by subsampling is a technique used to enhance privacy guarantees in differentially private mechanisms by randomly selecting a subset of the data before applying the privacy mechanism. This subsampling process can lead to a reduction in the privacy cost, allowing for better utility while preserving privacy. 
We can show that the Poisson subsampled Gaussian mechanism with sensitivity 1, noise scale $\sigma$, and subsampling rate $q$ has the PRV $Y = \log( P(o)/Q(o) ), o \sim P$, where $P = (1-q) \N(0, \sigma^2) + q \N(1, \sigma^2)$ and $Q = \N(0, \sigma^2)$, and $P(\cdot), Q(\cdot)$ are the density functions of $P, Q$. With the PRV of subsampled Gaussian mechanism as well as the PRV accountant, we can now efficiently and tightly track the privacy costs for DP-ICL.

\subsection{text classification via RNM-Gaussian}
\label{appendsub:analysis_rnm}

\begin{theorem}
\label{thm:RNM-Gaussian}
The mechanism RNM-Gaussian $\M_\sigma$ from Section \ref{sec:dpicl_class} is $(\eps, \delta)$-DP with $\sigma = 2 \sqrt{\log (1.25 / \delta)} / \eps$. 
\end{theorem}
\begin{proof}
Note that $\M_\sigma$ can be broken down into applying the $\argmax$ operator on a noisy histogram, which is generated by adding Gaussian noise to each dimension of the original histogram. The Gaussian mechanism is known to satisfy $(\eps, \delta)$-DP with $\sigma = \Delta \sqrt{2 \log (1.25 / \delta)} / \eps$ \citep{dwork2014algorithmic}, where $\Delta := \sup_{D \sim D'} \norm{ f(D) - f(D') }$ represents the global sensitivity of the underlying aggregation function $f$. In our case, $f$ calculates the original voting histogram. As each exemplar-query prediction may alter two counts (increasing one and decreasing the other), the sensitivity $\Delta$ is $\sqrt{2}$. The overall privacy guarantee is then derived from the post-processing property of differential privacy.
\end{proof}

\subsection{Embedding Space Aggregation (ESA)}
\label{appendix:priv-ESA}

\begin{theorem}
The Step 2 to Step 9 in Alg. \ref{alg:dpemb} is $(\eps, \delta)$-DP with $\sigma = 2 \sqrt{\log (1.25 / \delta)} / \eps$.  
\end{theorem}
\begin{proof}
Note that each embedding output by the text-to-embedding model has $\ell_2$ norm to be 1. Hence, the referred steps are essentially the same as Gaussian mechanism with $\ell_2$ sensitivity 1. 
The last step of releasing the public candidate sentence that has the maximum cosine similarity can be regarded as the post-processing step and hence does not affect the overall privacy guarantee. 
\end{proof}

\textbf{Tracking Privacy Loss with PRV Accountant for Gaussian mechanism.}
To better keep track of the privacy cost for RNM-Gaussian and ESA, we use the most recent advances in privacy cost accounting based on the notion of the Privacy Loss Random Variable (PRV) \citep{dwork2016concentrated}. The PRV accountant was introduced by \citet{koskela2020computing} and later refined in \citet{koskela2021computing, gopi2021numerical}. 
For any DP-algorithm, one can easily compute its $(\eps, \delta)$ privacy guarantee based on the distribution of its PRV. The key property of PRVs is that, under (adaptive) composition, they simply add up; the PRV $Y$ of the composition $\M = \M_1 \circ M_2 \circ \dots \circ M_k$ is given by $Y = \sum_{i=1}^k Y_i$, where $Y_i$ is the PRV of $\M_i$. Therefore, one can then find the distribution of $Y$ by convolving the distributions of $Y_1, Y_2, \dots, Y_k$. Prior works \citep{koskela2021computing, gopi2021numerical} approximate the distribution of PRVs by truncating and discretizing them, then using the Fast Fourier Transform (FFT) to efficiently convolve the distributions.

\subsection{Language Generation via Keyword Space Aggregation (KSA)}
\label{appendix:priv-KSA}

R{\'e}nyi differential privacy (RDP) is a variant of the standard $(\eps, \delta)$-DP that uses R{\'e}nyi-divergence as a distance metric between the output distributions of $\M(D)$ and $\M(D')$, which is particularly useful in training differentially private machine learning models. 

\begin{definition}[R{\'e}nyi Differential Privacy \citep{mironov2017renyi}]
\label{def:RDP}
We say that a mechanism $\M$ is $(\alpha, \eps_{\M}(\alpha))$-RDP with order $\alpha \in (1, \infty)$ if for every dataset pair $D, D'\in \MS(\cX)$ such that $d(D, D') = 1$, we have: 
\begin{align}
D_{\alpha}\left(\M(D) \|  \M\left(D^{\prime}\right)\right) :=\frac{1}{\alpha-1} \log \E_{o \sim \M\left(D^{\prime}\right)}\left[\left(\frac{\mu_{\M(D)}(o)}{\mu_{\M\left(D^{\prime}\right)}(o)}\right)^{\alpha}\right] \leq \eps_{\M}(\alpha)
\end{align}
where $\mu_\M(\cdot)$ denotes the density function of $\M$'s distribution. 
\end{definition}

Another useful relaxation of the RDP definition is approximate RDP. 

\begin{definition}[Approximate RDP \citep{bun2016concentrated, zhu2022adaptive}]
\label{def:approximate-RDP}
We say a randomized algorithm $\M$ is $\delta$-approximately $(\alpha, \eps_{\M}(\alpha))$-RDP with order $\alpha \geq 1$, if for all neighboring dataset $D, D'$, there exist events $E$ (depending on $\M(D)$) and $E'$ (depending on $\M(D')$) such that $\Pr[E] \geq 1-\delta$ and $\Pr[E'] \geq 1-\delta$, and $\forall \alpha \geq 1$, we have 
\begin{align}
D_\alpha \left(\mathcal{M}(D)|E~\|~ \mathcal{M}\left(D^{\prime}\right)|E^{\prime} \right) \leq \eps_{\M}(\alpha)
\end{align}
\end{definition}

For both methods of KSA-JEM and KSA-PTR, we use RDP and approximate RDP for a tighter measure of the privacy cost under composition. 
After we obtain the (approximate) RDP guarantee for the overall algorithm, we can then convert the privacy guarantee back into the standard DP definition. 
We refer the readers to \cite{bun2016concentrated} and \cite{mironov2017renyi} for the composition and conversion formula for RDP and approximate RDP. 
In the following, we state the privacy guarantee of individual building blocks for private prompt generation and selection in terms of (approximate) RDP. 

We then introduce the exponential mechanism \citep{McSherry2007MechanismDV}, one of the most famous and frequently used DP mechanisms. The exponential mechanism takes a utility function $q: \MS \times \cY \to \R$ and can be thought of as evaluating how good $q(D, y)$ is for an outcome $y \in \cY$ on dataset $D$. 

\newcommand{\EM}{\texttt{EM}}

\begin{definition}[Exponential Mechanism]
Let $\EM_q: \MS \to \cY$ be a mechanism where for all outputs $y \in \cY$ we have
$$
\Pr[\EM_q(D) = y] \propto \exp \left(\frac{\eps}{2 \Delta(q)} q(D,y) \right)
$$
where $\Delta(q)$ is the sensitivity of the quality score, i.e. for all neighboring inputs $D,D'$ we have $\sup_{y \in \cY} |q(D, y) - q(D',y)|\leq \Delta(q)$
\end{definition}

Furthermore, \cite{durfee2019practical} shows that
adding Gumbel noise to each output's utility and releasing the output with the highest noisy utility score is equivalent to using the exponential mechanism. 

\begin{theorem}[\cite{bun2016concentrated}]
\label{thm:priv-EM}
    The exponential mechanism is $\eps$-DP, and $(\alpha, \eps_{\EM}(\alpha))$-RDP s.t.
    \begin{align*}
        \eps_{\EM}(\alpha) := 
        \min \left( \frac{\alpha}{2} \eps^2,
        \frac{1}{\alpha-1} \log \left(
        \frac{\sinh(\alpha \eps) - \sinh((\alpha-1) \eps)}{\sinh(\eps)} 
        \right) \right)
    \end{align*}
\end{theorem}

\subsubsection{KSA via Joint Exponential Mechanism}

\begin{theorem}
    Alg. \ref{alg:joint} is $\eps$-DP, and $\eps_{EM}(\alpha)$-RDP. 
\end{theorem}
\begin{proof}
Alg. \ref{alg:rnm-find-k} is an Exponential mechanism on the domain space of positive integers $k = 1, 2, \ldots,$ where the utility of $k$ is $d_k := \bH_{(k)} - \bH_{(k+1)}$. The sensitivity of $d_k$ is 2. Hence, the DP and RDP guarantee follows from the privacy guarantee of exponential mechanism in Theorem \ref{thm:priv-EM}. 
\end{proof}

\subsubsection{KSA via Propose-Test-Release (KSA-PTR)}

\begin{theorem}
    Alg. \ref{alg:rnm-find-k} is $\eps$-DP, and $\eps_{EM}(\alpha)$-RDP. 
\end{theorem}
\begin{proof}
Alg. \ref{alg:rnm-find-k} is an Exponential mechanism on the domain space of positive integers $k = 1, 2, \ldots,$ where the utility of $k$ is $d_k := \bH_{(k)} - \bH_{(k+1)}$. The sensitivity of $d_k$ is 2. Hence, the DP and RDP guarantee follows from the privacy guarantee of exponential mechanism in Theorem \ref{thm:priv-EM}. 
\end{proof}

\begin{theorem}
    Alg. \ref{alg:ptr} is $\delta$-approximate $\frac{\alpha}{2 \sigma^2}$-RDP. 
\end{theorem}
\begin{proof}
    Releasing the noisy threshold $\widehat d_k$ is $\frac{\alpha}{2 \sigma^2}$-RDP. 

    If $d_k > 2$, then releasing the exact top-$k$ tokens has no privacy cost, as its local sensitivity is 0. 

    If $d_k \le 2$, then if $\widehat d_k \le 2$, the program terminates and there's no privacy cost. 

    If $d_k \le 2$, the failure probability
    \begin{align*}
        \Pr[\widehat d_k > 2] 
        &= \Pr[ \max(2, d_k) + \N(0, 4\sigma^2) - \Phi(1-\delta; 0, 2\sigma) > 2 ] \\
        &= \Pr[ 2 + \N(0, 4\sigma^2) - \Phi(1-\delta; 0, 2\sigma) > 2 ] \\
        &= \Pr[ \N(0, 4\sigma^2) - \Phi(1-\delta; 0, 2\sigma) > 0] \\
        &= \delta
    \end{align*}
\end{proof}

\subsubsection{Privacy amplification by subsampling for approximate RDP}

\begin{theorem}
    If $\M$ is $\delta$-approximate $\eps_\M(\alpha)$-RDP, then $\M \circ \poisson$ with subsampling rate $q$ is $\delta q$-approximate $\eps_{\M \circ \poisson}(\alpha)$-RDP, where $\eps_{\M \circ \poisson}(\alpha)$ is the tightest possible amplification bound for any mechanism that is $\eps_\M(\alpha)$-RDP with subsampling rate $\frac{q(1-\delta)}{1-q\delta}$. 
\end{theorem}
\begin{proof}
Consider $D := D' \cup \{z\}$, $D, D'$ are neighboring datasets. 
Denote $S \subseteq D'$, and let $\gamma_S$ the probability of sampling $S$. 
Denote $\mu_S := \M(S)$. 

\begin{align}
    \M( \poisson(D') ) &= \sum_{S \subseteq D'} \gamma_S \mu_S \\
    \M( \poisson(D) ) &= \sum_{S \subseteq D'} \gamma_S \left( (1-q) \mu_S + q\mu_{S \cup z} \right) \\
    &= (1-q) \sum_{S \subseteq D'} \gamma_S \mu_S + q \sum_{S \subseteq D'} \gamma_S \mu_{S \cup z}
\end{align}

For any pair of $S, S \cup z$, denote event $E_{S}, E_{S \cup z}$ s.t. $D_\alpha(\mu_S | E_{S} \| \mu_{S \cup z} | E_{S \cup z}) \le \eps_\M(\alpha)$ and $\Pr[E_{S}] \geq 1-\delta$ and $\Pr[E_{S \cup z}] \geq 1-\delta$. Hence, we can rewrite $\M( \poisson(D') )$ and $\M( \poisson(D) )$ as 
\begin{align*}
    \M( \poisson(D') ) 
    &= (1-q) \sum_{S \subseteq D'} \gamma_S \mu_S + q \sum_{S \subseteq D'} \gamma_S \mu_S \\ 
    &= (1-q) \sum_{S \subseteq D'} \gamma_S \mu_S + q \sum_{S \subseteq D'} \gamma_S \left( (1-\delta) \mu_S|E_S + \delta \mu_S|\bar E_S \right) \\
    &= (1-q) \sum_{S \subseteq D'} \gamma_S \mu_S + q(1-\delta) \sum_{S \subseteq D'} \gamma_S \mu_S|E_S + q \delta \sum_{S \subseteq D'} \gamma_S \mu_S|\bar E_S
\end{align*}

\begin{align*}
    \M( \poisson(D) ) &= (1-q) \sum_{S \subseteq D'} \gamma_S \mu_S + q \sum_{S \subseteq D'} \gamma_S \mu_{S \cup z} \\
    &= (1-q) \sum_{S \subseteq D'} \gamma_S \mu_S + q(1-\delta) \sum_{S \subseteq D'} \gamma_S \mu_{S \cup z}|E_{S\cup z} + q \delta \sum_{S \subseteq D'} \gamma_S \mu_{S \cup z}|\bar E_{S \cup z}
\end{align*}

Hence, there exists event $E_D, E_{D'}$ s.t. $\Pr[E_{D}] \geq 1-q \delta$ and $\Pr[E_{D'}] \geq 1-q \delta$, and 
\begin{align*}
    \M( \poisson(D') )|E_D &= (1-q) \sum_{S \subseteq D'} \gamma_S \mu_S + q(1-\delta) \sum_{S \subseteq D'} \gamma_S \mu_S|E_S \\
    \M( \poisson(D) )|E_{D'} &= (1-q) \sum_{S \subseteq D'} \gamma_S \mu_S + q(1-\delta) \sum_{S \subseteq D'} \gamma_S \mu_{S \cup z}|E_{S\cup z}
\end{align*}
Hence
\begin{align}
    D_\alpha \left( \M( \poisson(D) )|E_D~\|~\M( \poisson(D') )|E_{D'} \right) 
\end{align}
has privacy amplification with subsampling rate $\frac{q(1-\delta)}{1-q\delta}$. 
\end{proof}

% \subsection{Privacy Analysis for Limited Domain in Renyi DP}

% \begin{theorem}
%     Limited domain that releases $k$ elements satisfy $(k \eps_{EM}(\alpha), \delta)$-approximate RDP. 
% \end{theorem}

\newpage

% \vspace{-3mm}
\section{DP-ICL Enables Private Prediction}
\label{appendsec:privateprediction}
% \section{Discussion}
% \vspace{-2mm}

% \tong{ Also mention fine-tuning need to hold different model weights for different tasks, while dp-icl only need one model. }

Our work represents a major departure from prior work on DP LLMs in that we consider private \emph{prediction} rather than private training.
A line of recent work~\citep{Li2021LargeLM, Yu2021DifferentiallyPF, bu2022differentially, he2023exploring} has proposed fine-tuning pre-trained models on downstream tasks with differentially private stochastic gradient descent (DP-SGD)~\citep{abadi2016deep}.
Despite ample research into DP LLMs and the growing industry demand for solutions to augment LLMs with proprietary data~\citep{kuchaiev2019nemo, nvidia_2023}, a number of key challenges remain for DP LLMs that we seek to address by considering \emph{private prediction}.

% \textbf{Private training degrades utility.}
% The bulk of evaluation done in prior work on DP LLMs is done at unrealistic privacy budgets ($\eps>3$, e.g. $\eps=50$~\citep{majmudar2022differentially}) that are not in line with industry standards~\citep{DworkKM19}.
% We consider $\eps \in [1,8]$ and provide competitive results with non-private ICL even for the conservative privacy budget of $\eps=1$.
% The compromise on privacy budgeting is necessary in private training because the threat model for private training is often overly strong and therefore sacrifices utility. 
% Specifically, private training operates under an overly strong threat model that assumes all downstream users can collude and directly observe the trained model.
% We instead follow the threat model of~\citet{gaboardi2016psi} that makes more realistic assumptions about adversaries' information and resources. 
% Specifically, we assume that downstream users cannot view the model, do not share their results with each other and do not collude in coordinated attacks on individual training samples. 
% This allows each user to independently spend their privacy budget, leading to improved utility without compromising data privacy. 
% We note that \emph{even considering the low probability of downstream users colluding to deanonymize private exemplar data, we still assume an attacker that can observe up to $k=10,000$ query answers}.
% \textbf{Our method does not compromise utility when evaluated with conservative privacy budgets on challenging datasets.}

\textbf{Private training makes training harder.}
Fine-tuning with DP-SGD requires adopting entirely new hyperparameters and shifting existing hyperparameters to be radically different from non-private training~\citep{Li2021LargeLM}.
Performing this additional hyperparameter tuning can take hundreds of trials.
DP-SGD uses per-example gradient clipping to bound the sensitivity of individual datapoints.
Materializing per-example gradients can increase the memory consumption of training by an order of magnitude~\citep{bu2022differentially} and slow down training. 
Although recent methods have been proposed for efficient hyperparameter tuning~\citep{panda2022dp, papernot2022hyperparameter}, efficient per-example gradient clipping~\citep{Li2021LargeLM}, and parameter-efficient fine-tuning~\citep{Yu2021DifferentiallyPF}, we emphasize that DP-SGD introduces challenging engineering and optimization problems that are a topic of ongoing research.
\textbf{Our method requires no hyperparameter tuning and is computationally efficient.}

\textbf{Private training is incompatible with black-box LLMs.}
Developers building on top of cloud-hosted LLMs such as OpenAI, Anthropic, or AWS Bedrock cannot implement the complex DP-SGD algorithm~\citep{aws_2023}. 
Organizations employing closed-source LLMs such as GPT-3+, Claude, or Bard cannot even access the weights for fine-tuning and may never be able to~\citep{openai2023gpt4}.
\textbf{Our method is compatible with any LLM API.}

\textbf{Private training does not allow flexible data editing.} 
Private training generates a single model that is inextricably tied to each datapoint in its training data.
This is at odds with the right to be forgotten mandated by GDPR~\citep{politico_2023}, that would require retraining the entire model to delete the influence of a private datapoint -an impracticality if not an outright impossibility when considering fine-tuning billion-parameter models.
By contrast, honoring the right to be forgotten with DP-ICL is as straightforward as just removing the individual's private data from the exemplar database. 
\textbf{Our method enables the right to be forgotten.}

% \vspace{-2mm}

\textbf{DP-ICL outperforms all previous DP-SGD methods on SST-2 benchmark (Table \ref{tab:compare}).} We also compare our results with current state-of-the-art differentially private stochastic gradient descent (DP-SGD) methods on \textbf{SST-2}.  The results illustrate an improvement over earlier methods. For instance, the enhancement at $\eps = 3$ is \textbf{1.2\%}, which translates to an over \textbf{20\%} reduction in relative error rate, thereby establishing a new SOTA in the field.\footnote{A minor discrepancy exists between our training data (sentence level) and the DP-SGD training data (phrase level) on the SST-2 dataset. Our training data is 10 times smaller than that of DP-SGD. However, the test data remains identical for both.} Moreover, by presenting the results with $\eps =\infty$, we notice that our performance gains do not correlate to the advanced large language model, where our upper bound is lower than other methods. That means DP-ICL exhibits less sacrifice when achieving a differential privacy guarantee.

\begin{table}[H]
    \caption{ \textbf{Comparison of DP-ICL and DP-SGD on the SST-2 Dataset.}   Our DP-ICL method demonstrates significantly lower performance degradation under privacy constraints of $\eps=\{3,8\}$.}
    \centering
      \resizebox{0.8\textwidth}{!}{%
    \begin{threeparttable}
    \begin{tabular}{@{}cllll@{}}
    \toprule
    \textbf{Model} & \textbf{Method}  & \textbf{$\eps=3$} (gap)   & \textbf{$\eps=8$} (gap)  & \textbf{$\eps=\infty$} \\ \midrule
    \multirow{4}{*}{   \begin{tabular}[c]{@{}c@{}}RoBERTa-large\\ \citep{Liu2019RoBERTaAR}\end{tabular}}  
    & DP-SGD \citep{Li2021LargeLM}             & 93.04 (\red{-3.16}) & 93.81 (\red{-2.39})   &  96.20  \\
    & DP-SGD \citep{Yu2021DifferentiallyPF}    & --                  & 95.30\tnote{*} (\red{-1.10}) &   96.40  \\
    & DP-SGD \citep{Bu2022DifferentiallyPB}    & 94.60 (\red{-0.90}) & 94.70 (\red{-0.80})  &   95.50   \\
    & DP-SGD \citep{he2023exploring}           & 94.23 (\red{-1.97}) & 94.87 (\red{-1.33})  &  96.20   \\ 
    \midrule
    \multirow{1}{*}{GPT-3 Davinci} 
     & DP-ICL (Ours)         &  \textbf{95.80}$_{0.21}$ (\green{+0.07})   &    \textbf{95.83}$_{0.22}$ (\green{+0.10})  &   95.73  (4-shot)  \\
    \bottomrule
    \end{tabular}
    \begin{tablenotes}
    \item[*] Result present in \citep{Yu2021DifferentiallyPF} is $\eps=6.7$.
    \end{tablenotes}
    \end{threeparttable}
        }
         \vspace{-4mm}
    \label{tab:compare}
\end{table}

% Also show the error rate. 

However, DP-ICL in \cref{tab:compare} is capable of responding to a maximum of 10,000 queries, whereas DP-SGD has no such query limit and can answer an arbitrarily large number of queries. We defer to practitioners to evaluate these factors when selecting the most appropriate algorithm for real-world applications.

\newpage

\newpage
\section{ Details of Experiments Setup }
\label{append:setup}

In this appendix, we present more details of our experiment setup, including the text classification task (Appendix \ref{appendsub:setupt}) as well as the language generation tasks (Appendix \ref{appendsub:setupla}).

\subsection{Text Classification}
\label{appendsub:setupt}
\textbf{Task.} 
Following \citet{Zhang2022ActiveES},
we employ the template summarized in Table \ref{tab:template} to carry out experiments on text classification tasks. We configure the logit bias to 100 via the GPT-3 API to ensure that the output token belongs to one of the predefined labels (e.g., positive and negative). We then select the label with the highest probability. We use a value of 0.0 for the temperature.

\begin{table}[H]
    \caption{ The template prompts for our experiments. }
         % \vspace{-4mm}
% \vspace{-1mm}
    \centering
    \resizebox{0.7\textwidth}{!}{%
    \begin{tabular}{@{}ccc@{}}
    \toprule
    \textbf{Dataset} & \textbf{Template} & \textbf{Labels}\\ \midrule
   SST-2     &   Review: \{text\}  Sentiment: \{label\} &  Positive, Negative \\ \midrule
   \multirow{2}{*}{  Amazon }  
   &  Title: \{title\}  Review: \{review\}   &  \multirow{2}{*}{Positive, Negative }  \\  
   &  Sentiment \{label\}   &  \\ \midrule
   \multirow{2}{*}{AGNews}    &   \multirow{2}{*}{Article: \{text\}  Answer: \{label\}} &   World, Sports,\\
   & & Business, Technology \\ \midrule
    \multirow{4}{*}{TREC}    &  Classify the questions based on whether &  \multirow{4}{*}{   \begin{tabular}[c]{@{}c@{}} Number, Location, \\ Person, Description,  \\ Entity, Abbreviation \end{tabular}}  
    \\  &  their answer type is a Number, Location, & 
    \\  & Person, Description, Entity, or Abbreviation. &  
    \\  & Question: \{text\}  Answer Type: \{label\} & \\
    \bottomrule
    \end{tabular}%
    }
    \label{tab:template}
    \end{table}

\textbf{Dataset.} Our experiments are conducted on four datasets, the details of which are presented in Table \ref{tab:datainfo}. Given that the training data size for each dataset is fewer than 10,000, we set $\delta$ to $1e^{-5}$.

\begin{table}[H]
    \caption{ Information about the Text Classification Dataset}
     % \vspace{-4mm}
    \centering
    \begin{tabular}{@{}cccccc@{}}
    \toprule
    Dataset & Task  & \# of classes & \# of exemplars & \# of test data& avg. length \\ \toprule
   SST-2  &   Sentiment cls.  & 2 & 6,920 & 872 &  37.8  \\ \midrule
   Amazon &   Sentiment cls.  & 2 & 8,000 & 1,000 &  78.5  \\ \midrule
   AGNews &   Topic cls.      & 4 & 8,000 & 1,000 &  19.3  \\ \midrule
   TREC   &   Question cls.   & 6 & 5,452 & 500  &  10.2  \\ 
    \bottomrule
    \end{tabular}%
    \label{tab:datainfo}
    \end{table}

% \textbf{Model.} Following \citet{Zhang2022ActiveES}, we conduct experiments in four datasets shown in Table \ref{tab:datainfo}. 

\subsection{Language Generation}
\label{appendsub:setupla}

We then provide the detailed setup for document questions answering and dialog summarization.

\subsubsection{Document Questions Answering} 
\label{appendsub:setupdqa}
\textbf{Task.} For document  questions answering task, we use data from Privacy Preserving Federated Learning Document VQA (PFL-DocVQA) competition \footnote{The web link is \url{https://benchmarks.elsa-ai.eu/?ch=2}}. 
The primary objective is to create privacy-preserving approaches for fine-tuning multi-modal models, for both vision and language, to improve document understanding. 
These documents often contain sensitive or confidential information. As our research is centered on language generation, we skip the vision module and assume that accurate OCR tokens have been extracted from the document images. We leverage the language model to answer the posed questions, and an example of our evaluation is presented as follows: 

% \begin{mdframed}[backgroundcolor=gray!5, linecolor=black, linewidth=1pt]
% \textbf{Extracted OCR tokens from image:}
% Page,2,of,2,DUNORICE,REISSUE,Send,Payment,To:,
% WSIL-TV,Invoice,Invoice,Date,Invoice,Month,Invoice,Period,WSIL,PO,Box,1001,12970-2B,11/11/19,November,2019,10/28/19,10/28/19,HARRISBURG,Quincy,...

% \textbf{Question:} What is the identification number assigned to the invoice in the document?

% \textbf{Answer the question with short term:}  12970-2B

% \textbf{Extracted OCR tokens from image:}
% ORDER,WORKSHEET,Rep,Order#,10447675,Ver#,1,Status,New,Traffic,Order#,Printed:,10/17145:14PM,Last,Received:,10/17/45:09PM,Showing,Buylines:,All,Rep,Changes,orl,Last,Station,Changes,2of,2,EC'd,Yes,Station,WPGH-TV,PITTSBURGH,,PA,Advertiser,(NW8O,POLITICAL,Product,CMTE2ELECTSOLOBAY,Estimate#,3872,Buyer,Kristal,Murray,Agency,(BUYT)BUYINGTIME,LLC,650MASSACHUSETISAVENW,SUITE210,WASHINGTON,,,DC20001,Agency,C/P1/P2/E,618/626/3872,Flight,Dates,1021/14-1027/4
% \textbf{Question:} Could you share the ID associated with the document? 

% \textbf{Answer the question with short term:} 
% \end{mdframed}

\begin{mdframed}[backgroundcolor=gray!5, linecolor=black, linewidth=1pt]

% First Section
\noindent \textbf{Extracted OCR tokens from image:}\\
Page,2,of,2,DUNORICE,REISSUE,Send,Payment,To:,WSIL-TV,Invoice,\dots

\vspace{-0.5em}
\noindent \textbf{Question:} What is the identification number assigned to the invoice in the document?

\vspace{-0.5em}
\noindent \textbf{Answer the question with short term:}  12970-2B

% Second Section
% \vspace{2em}
\noindent \textbf{Extracted OCR tokens from image:}\\
ORDER,WORKSHEET,Rep,Order\#,10447675,Ver\#,1,Status,New,Traffic,\dots

\vspace{-0.5em}
\noindent \textbf{Question:} Could you share the ID associated with the document?

\vspace{-0.5em}
\noindent \textbf{Answer the question with short term:}

\end{mdframed}

\textbf{Dataset and Metrics.} We use the training data contains 221,329 data and evaluate the performance on 100 data that are randomly selected from validation dataset. 
Therefore, we set the $\delta$ for DP guarantee to (1/250,000). 
We use three evaluation metrics in this task: ROUGE-1, BLEU, and Levenshtein similarity. 
Specifically, ROUGE-1 measures the overlap of unigrams between the generated and reference texts, and BLEU checks for matching words but considers them in n-grams. 
The Levenshtein similarity metric is derived from the Levenshtein distance, which measures the minimum number of single-character edits needed to transform one string into another.

% Due to the potentially extensive length of the OCR-extracted tokens, we refrained from using the paid GPT API for conducting our experiments. We anticipate carrying out more resource-intensive tests once additional budget becomes available. In the interim, we have opted for the open-source model OpenLLaMA-13B, as cited, for our evaluation.

\textbf{Model.} Due to the long OCR-extracted tokens, we avoided using the paid GPT API. We plan to conduct more extensive experiments if we have increased funding. For the current evaluation, we used the open-source model OpenLLaMA-13B \cite{openlm2023openllama}.

\subsubsection{Dialog Summarization} 
\label{appendsub:setupsum}

\textbf{Task.} Then we provide the detailed experiment setup for dialog summarization, where we use the SAMsum dataset \citet{gliwa-etal-2019-samsum}. This dataset has conversations between people, some of which may be private. We use the prompt to guide the task, which is as follows:

\begin{mdframed}[backgroundcolor=gray!5, linecolor=black, linewidth=1pt]

\noindent \textbf{Dialogue:}
\vspace{-0.5em}
\begin{itemize}[left=0em,label=--,nosep]
  \item Sandra: Are you sleeping?
  \item Matija: Nope
  \item Sandra: How so?
  \item Matija: It is Friday :D
  \item Sandra: So you are partying?
  \item Matija: File photo :D
  \item Sandra: You are in bed already?
  \item Matija: Of course :P
  \item Sandra: You are like old folks ;) What are your plans for tomorrow?
  \item Matija: Cleaning the apartment
  \item Sandra: Wow, nice
  \item Matija: Yea. Afterwards we are going to some restaurant, I have a lesson at 4 and we are meeting that bartender at 6
  \item \dots
\end{itemize}

\noindent \textbf{Summarize the above dialogue:} \\
Matija is in his bed on Friday nights. Tomorrow he is cleaning the apartment and afterwards they are going to the restaurant to have a lesson at 4 and meet that bartender at 6. 
\vspace{-0.5em}

\noindent \textbf{Dialogue:}
\vspace{-0.5em}
\begin{itemize}[left=0em,label=--,nosep]
  \item Joyce: Check this out!
  \item Joyce: <link>
  \item Michael: That's cheap!
  \item Edson: No way! I'm booking my ticket now!!
\end{itemize}

\noindent \textbf{Summarize the above dialogue:} \\

\end{mdframed}

When using keyword space aggregation by PTR, we change the \textit{``Summarize the above dialogue:''} to  \textit{``Summarize the above dialogue with the following word suggestions:''}.
When using keyword space aggregation by joint EM, we change the \textit{``Summarize the above dialogue:''} to  \textit{``Summarize the above dialogue with the following word suggestions ranked by their frequency from high to low:''}.

\textbf{Dataset and Metrics.} We use the training data containing 14,732 data and evaluate the performance on 100 test data. We set the $\delta$ to 5e-5. Following \citet{he2023exploring}, we use ROUGE-1, ROUGE-2, and ROUGE-L as the evaluation metrics. 

\textbf{Model.} We leverage the GPT-3 Davinci to conduct the experiments. Our methods can also apply to other large language models.

\newpage

\newpage
\section{Additional Evaluations on text classification}
\label{append:addeva_class}

In this appendix, we provide additional ablation studies for the text classification task. We discuss the impact of using DP-ICL with different queries (Appendix \ref{appsub:class_query}) and models (Appendix \ref{appsub:class_model}). We also examine how performance changes with varying numbers of in-context examples (Appendix \ref{appsub:class_shots}) and subsampling rates (Appendix \ref{appsub:class_subsample}).

\subsection{Effectiveness over various queries}
\label{appsub:class_query}
In this subsection, we conduct ablation experiments to learn the influence of varying the number of queries on performance, as illustrated in Figure \ref{Fig:ab}. The outcomes for the TREC dataset were obtained using the GPT-3 Davinci model, whereas results for all other datasets employed the GPT-3 Babbage model. We keep $\eps=3$ constant and use 10 ensembles. 
Our observations reveal that the performance degradation resulting from an increase in the number of queries remains negligible, up to 10,000, with a decrease of less than 2\%. 
Furthermore, it was observed that performance drops were more significant for the AGNews and TREC datasets compared to the remaining datasets.
We hypothesize that this drops attributed to the property of multiple classes (i.e., 4 and 6) within the AGNews and TREC datasets.

\begin{figure}[H]
  \centering
    \includegraphics[width=0.50\textwidth]{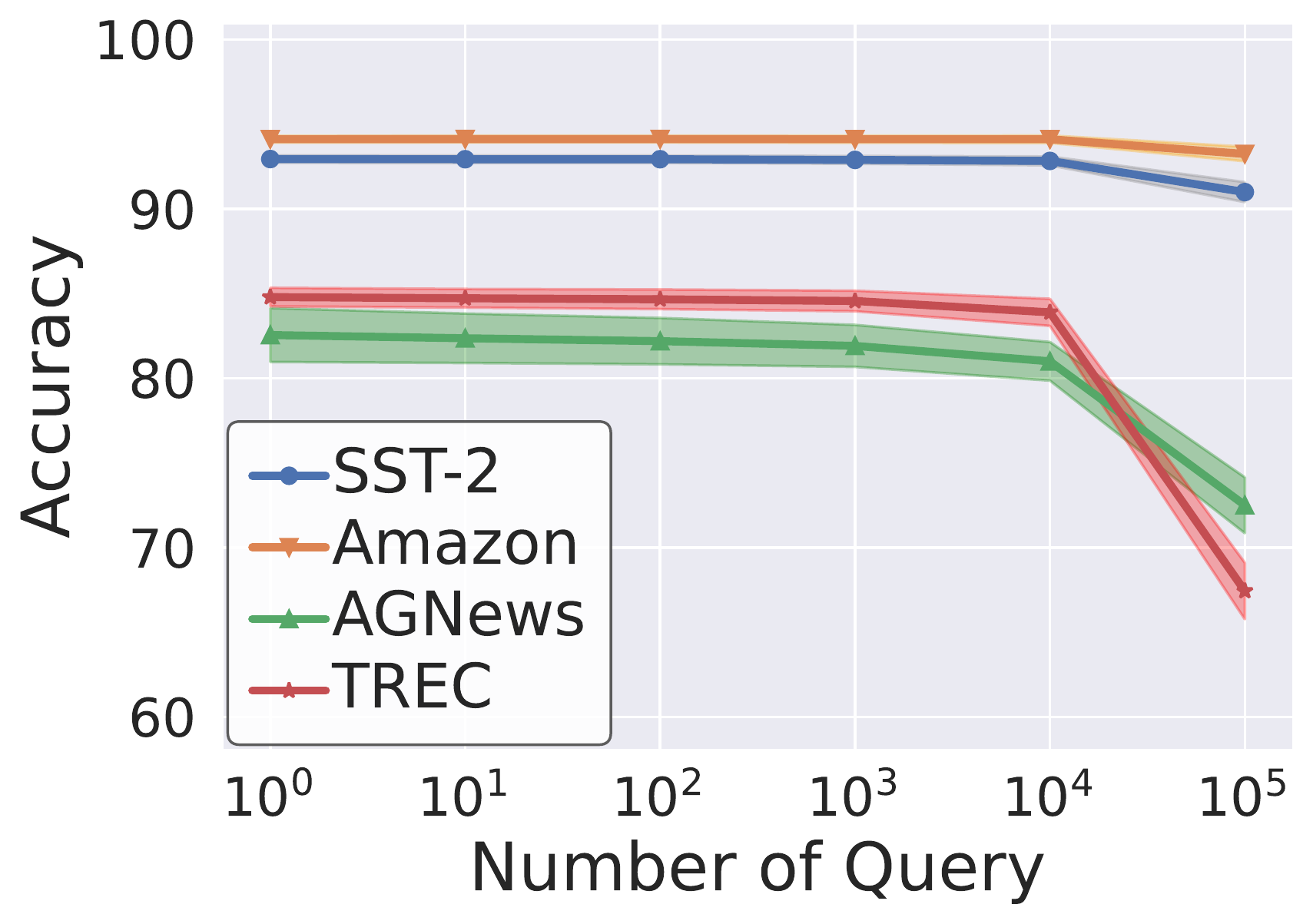} 
    \caption{Performance across numbers of the query.}
    \label{Fig:ab}
  \end{figure}
  
\subsection{Effectiveness of Various Models}
\label{appsub:class_model}

In this subsection, we explore various GPT-3 variants available through OpenAI's service. Specifically, we evaluate GPT-3 Ada (350M), GPT-3 Babbage (1.3B), GPT-3 Curie (6.7B), and GPT-3 Davinci (175B) on four text classification tasks, as detailed in Table \ref{tab:class_abl_model}. 
We configure the setting to be the same as our results in Table \ref{tab:main}. 

\begin{table}[H]
 \vspace{-4mm}
    \caption{ Results of DP-ICL for Text classification}
    \centering
    \resizebox{0.9\textwidth}{!}{%
    \begin{tabular}{@{}cccccccc@{}}
    \toprule
    \textbf{Dataset} & \textbf{Model} & $\eps=0$ (0-shot)   & $\eps=1$ & $\eps=3$ & $\eps=8$ & $\eps=\infty$ (Agg)& $\eps=\infty $\\ \midrule
   \multirow{4}{*}{SST-2} & Ada & 86.24     & 73.31 &   75.10 &   75.38 & 72.29 & 70.92 \\ 
                          & Babbage & 86.58 & 91.97 &   92.83 &   92.90 & 92.87 & 91.89 \\ 
                          & Curie & 91.51   & 94.03 &   94.86 &   94.94 & 95.07 & 94.03 \\ 
                          & Davinci & 94.15 & 95.11 &   95.80 &   95.83 & 95.73 & 95.49 \\ \midrule
   \multirow{3}{*}{Amazon}& Ada & 92.00     & 89.68 &   90.26 &   90.32 & 90.70 & 88.51\\ 
                          & Babbage & 93.80 & 93.83 &   94.10 &   94.12 & 94.10 & 93.58\\
                          & Curie & 95.90   & 95.37 &   95.67 &   95.69 & 95.58 & 95.28\\ \midrule
   \multirow{3}{*}{AGNews}& Ada & 37.50     & 65.60 &   70.51 &   71.22 & 71.72 & 60.23\\ 
                          & Babbage & 52.60 & 75.49 &   81.00 &   81.86 & 82.22 & 68.77 \\
                          & Curie & 62.40   & 80.00 &   81.88 &   82.06 & 81.80 & 78.31\\  \midrule
   \multirow{4}{*}{TREC}  & Ada & 20.40     & 23.72 &   25.41 &   25.86 & 26.28 & 22.62\\ 
                           & Babbage & 23.00& 24.48 &   26.36 &   26.26 & 26.32 & 27.00\\ 
                          & Curie & 29.20   & 37.95 &   41.96 &   42.50 & 42.24 & 36.57\\ 
                          & Davinci & 79.60 & 73.31 &   83.86 &   84.53 & 85.92 & 79.08\\           
    \bottomrule
    \end{tabular}%
    }
     \vspace{-4mm}
    \label{tab:class_abl_model}
    \end{table}

We observe consistent performance improvements as the model becomes more advanced—indicated by increased parameters—across all datasets. For example, in the SST-2 dataset, zero-shot performance increased from 86.24\% for Ada to 94.15\% for Davinci. Correspondingly, our DP-ICL results also improved, rising from 73.31\% to 95.11\% with $\eps=1$ on the SST-2 dataset. This suggests that more advanced models in the future can further enhance DP-ICL performance.
Interestingly, zero-shot performance can sometimes surpass in-context learning results, particularly for the Ada model on the Amazon dataset. This is caused by the influence of suboptimal in-context exemplars or limited in-context learning capabilities for small models. Therefore, estimating the utility of in-context learning (ICL) and zero-shot learning is essential before deploying DP-ICL.

\subsection{Effectiveness over the number of in-context exemplars.}
\label{appsub:class_shots}

In our primary evaluation in Table \ref{tab:main}, we set our analysis to four in-context exemplars. In this ablation study, we extend our evaluation to include a broader range of in-context exemplars: 1, 2, 4, 8, and 16. For these experiments, we maintain an ensemble size of 10.  Therefore, using more shots requires a higher subsampling rate, which increases noise in the voting histogram.

\begin{figure}[H]
  \centering
    \includegraphics[width=0.50\textwidth]{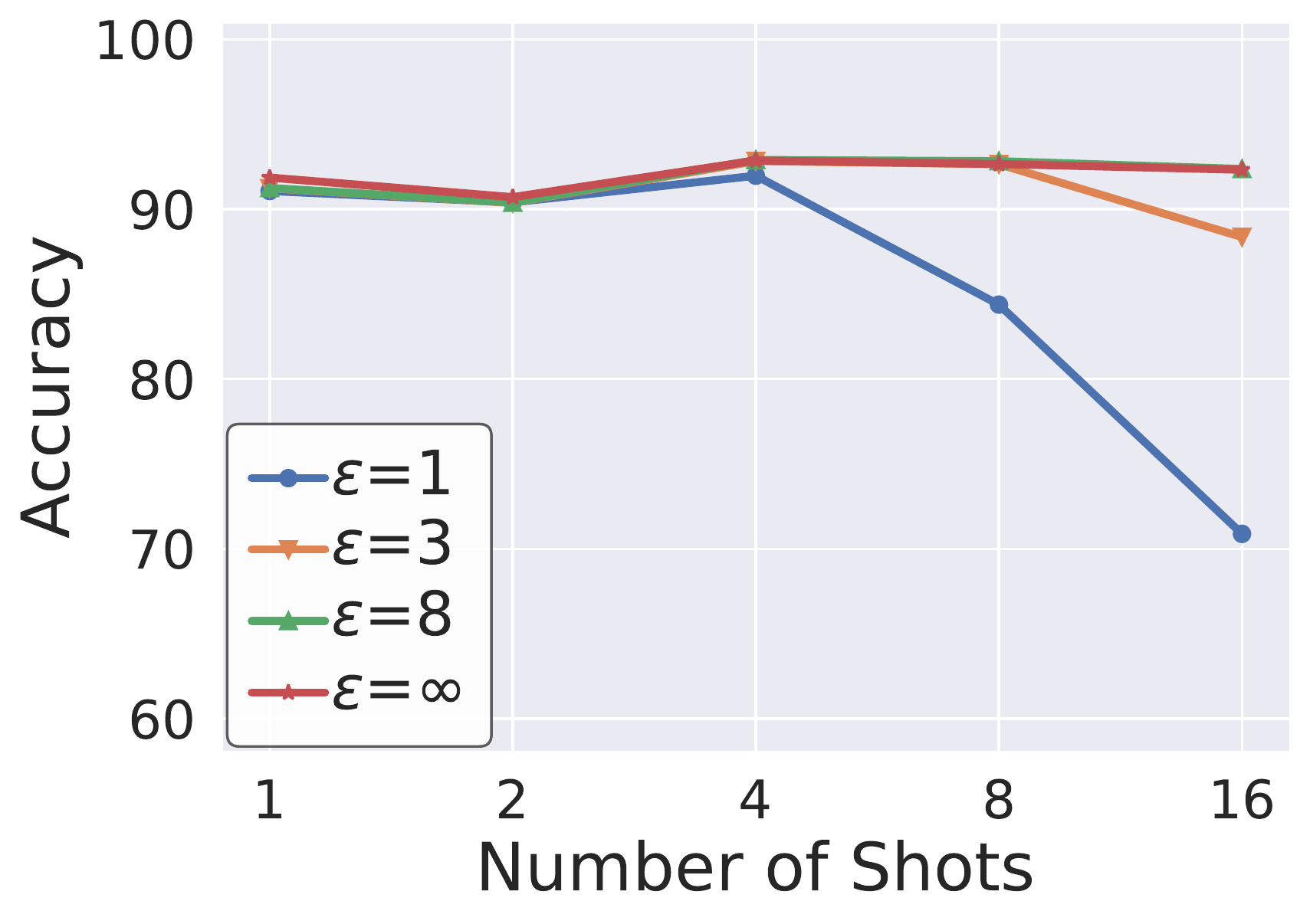} 
    \caption{Performance across numbers of the in-context exemplars.}
    \label{Fig:ab_ice}
  \end{figure}

In Figure \ref{Fig:ab_ice}, our evaluation reveals that performance remains stable for $\eps=\infty$. However, a significant performance degradation occurs for setups with 8 and 16 shots, experiencing drops of about 10\% and 20\%, respectively. This degradation is attributed to the introduction of additional noise.

\subsection{Effectiveness over subsampling rates}
\label{appsub:class_subsample}
To better understand the cost-utility trade-off in our proposed DP-ICL, we investigate the impact of subsampling strategies in Fig. \ref{fig:cost}. 
To save the computation cost, our evaluation is limited to test sets comprising 100 data for two datasets: SST-2 and AGNews. 
Our empirical results indicate that a subsampling rate of around $0.5 \times 10^{-2}$ for the exemplars achieved satisfactory performance, which are 10 samples post-subsampling. 
We further evaluate the estimated cost associated with querying the GPT-3 Babbage API, as illustrated in Fig. \ref{fig:cost} (Right). 
Notably, at a subsampling rate of  $0.5 \times 10^{-2}$, the cost is only $\$0.0945$ for 100 predictions on the SST-2 dataset.
Increasing the subsampling rate beyond this point does not significantly improve performance while resulting in higher costs.

\begin{figure}[H]
  \centering
  \begin{tabular}{cccc}
    \includegraphics[width=0.40\textwidth]{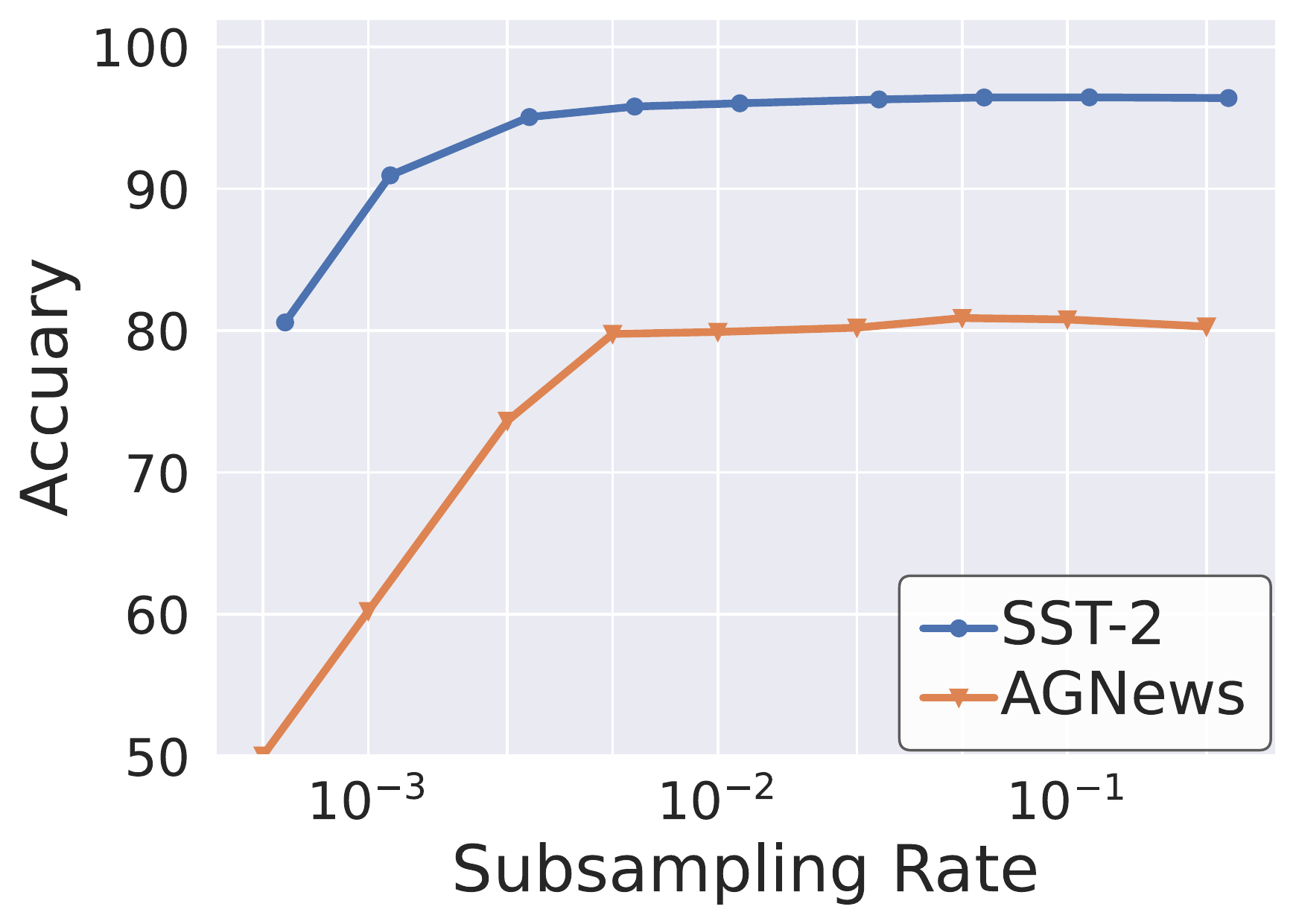} & 
   \includegraphics[width=0.40\textwidth]{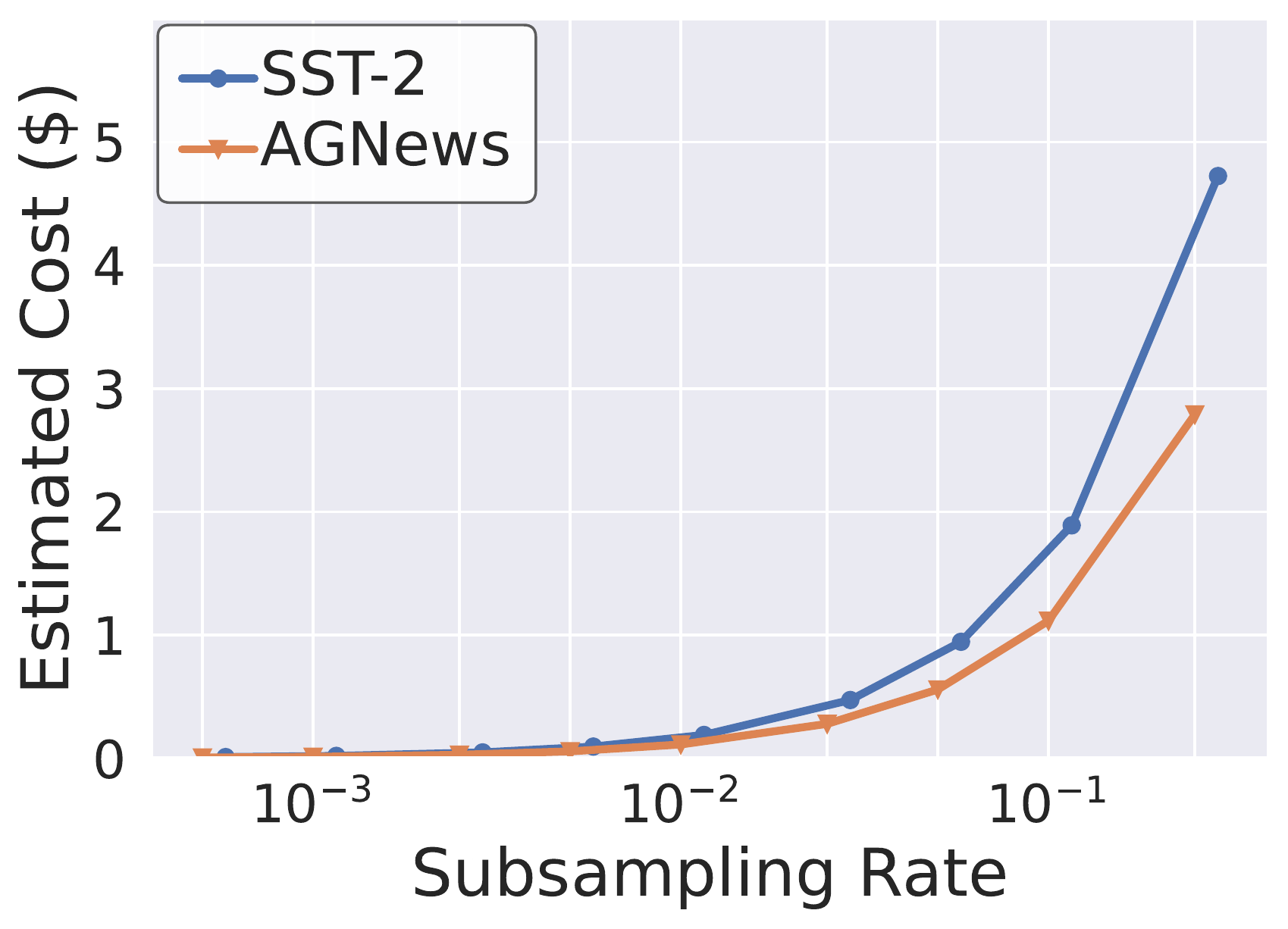} \\
  \end{tabular}
    \caption{Left: Performance across the subsampling rate. Right: Estimated GPT-3 Babbage API Cost of predicting 100 test samples. The size of the training set for SST-2 is 6,920, and for AGNews, it is 8,000. }
    \label{fig:cost}
  \end{figure}

\newpage
\newpage
\section{Additional Evaluations on Language generation}
\label{append_gen}

In this appendix, we present ablation studies for language generation tasks. Results for the document question-answering task are in  \cref{appsub:aba_lang_dqa}, and those for dialog summarization are in \cref{appsub:aba_lang_ds}.

\subsection{Document Question Answering task}
\label{appsub:aba_lang_dqa}

\textbf{Performance across varying numbers of queries (\cref{Fig:abl_query_dqa})} In document question answering tasks, we evaluate the keyword space aggregation via PTR method, as it consistently outperforms embedding space aggregation. Our analysis reveals that performance remains stable up to $10^6$ queries, allowing us to guarantee a strong performance for at least $10^5$ queries.

  \begin{figure}[H]
  \centering
  \begin{tabular}{cccc}
    \includegraphics[width=0.33\textwidth]{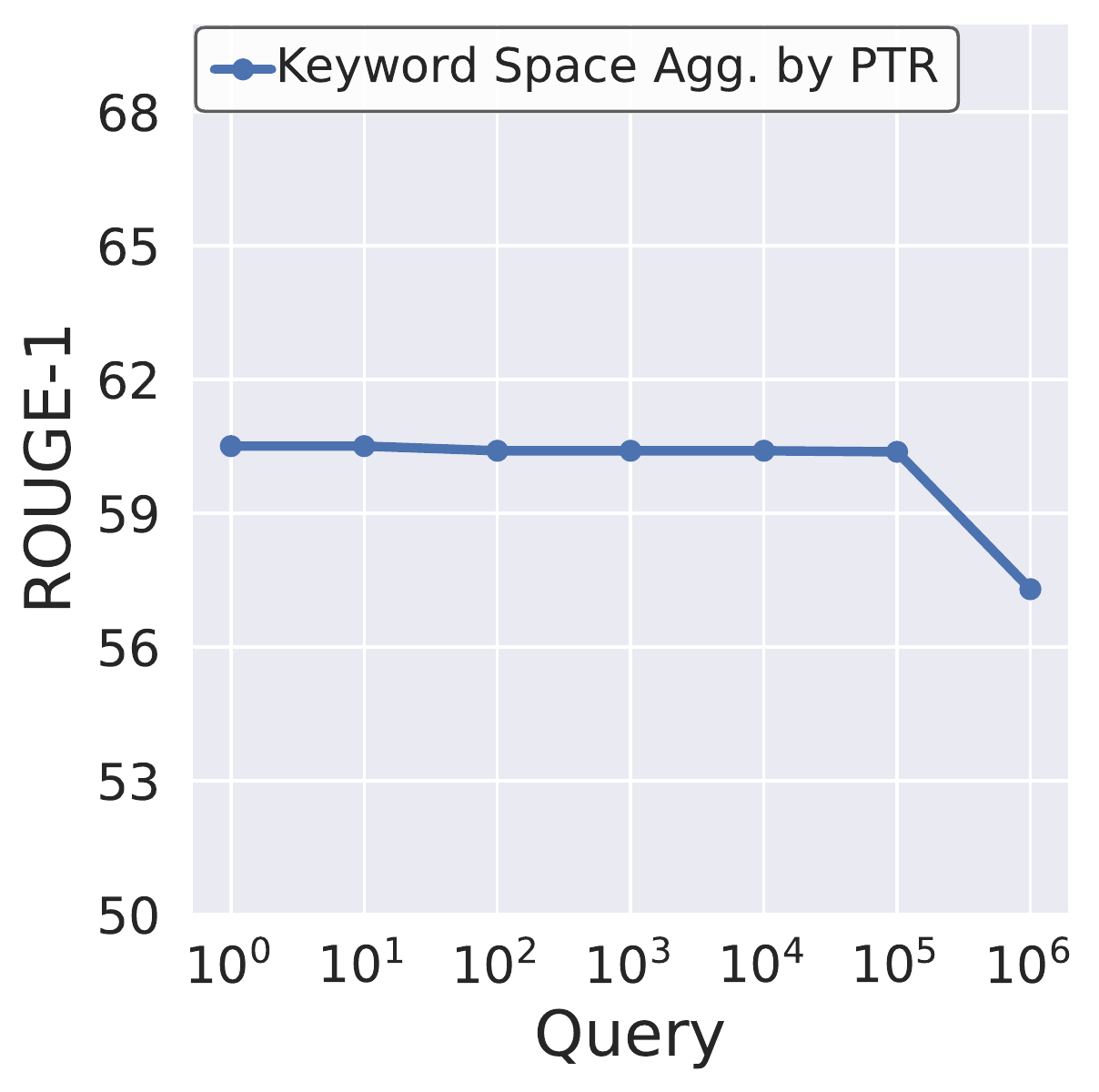} & 
     \includegraphics[width=0.33\textwidth]{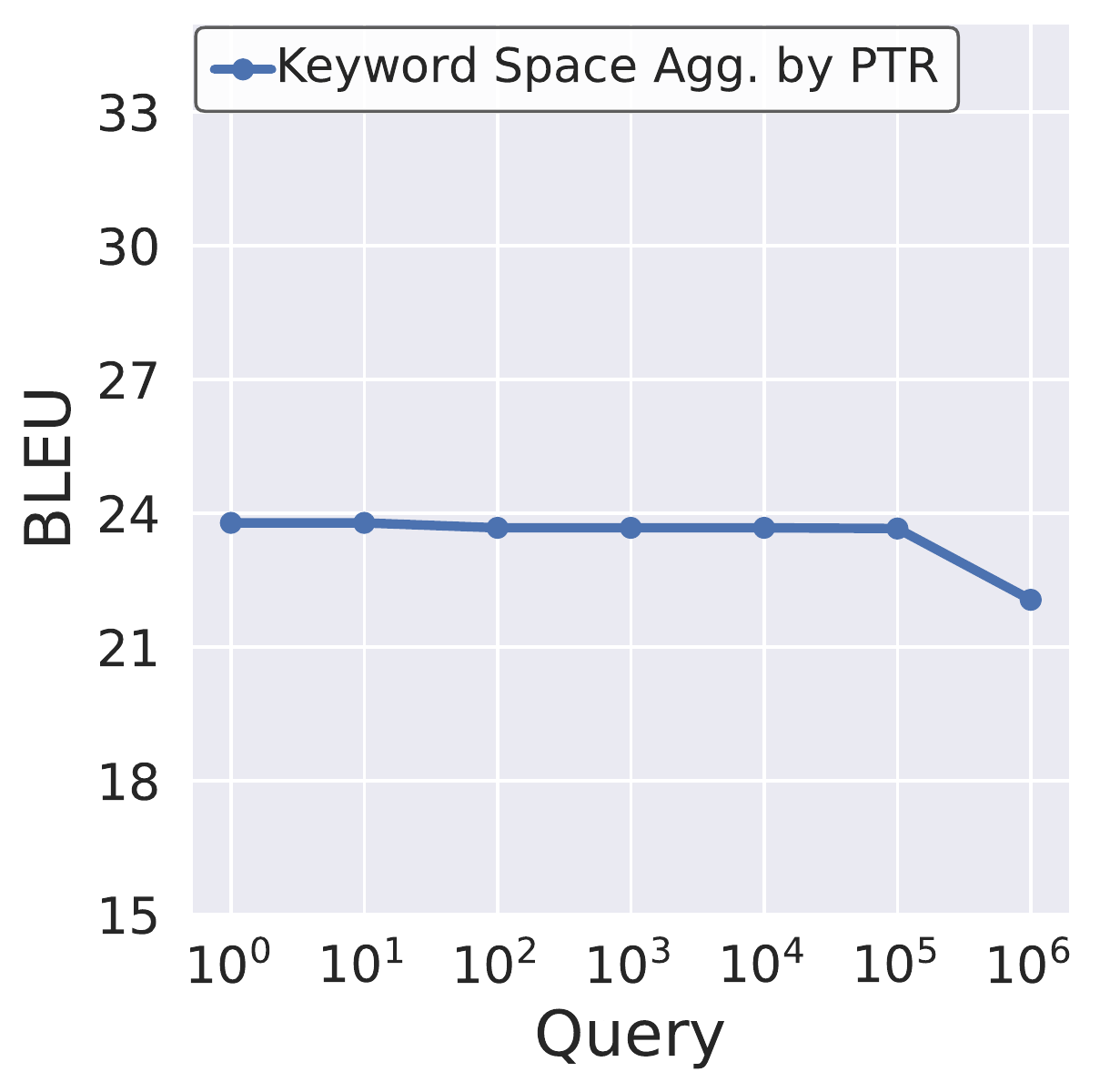} 
   \includegraphics[width=0.33\textwidth]{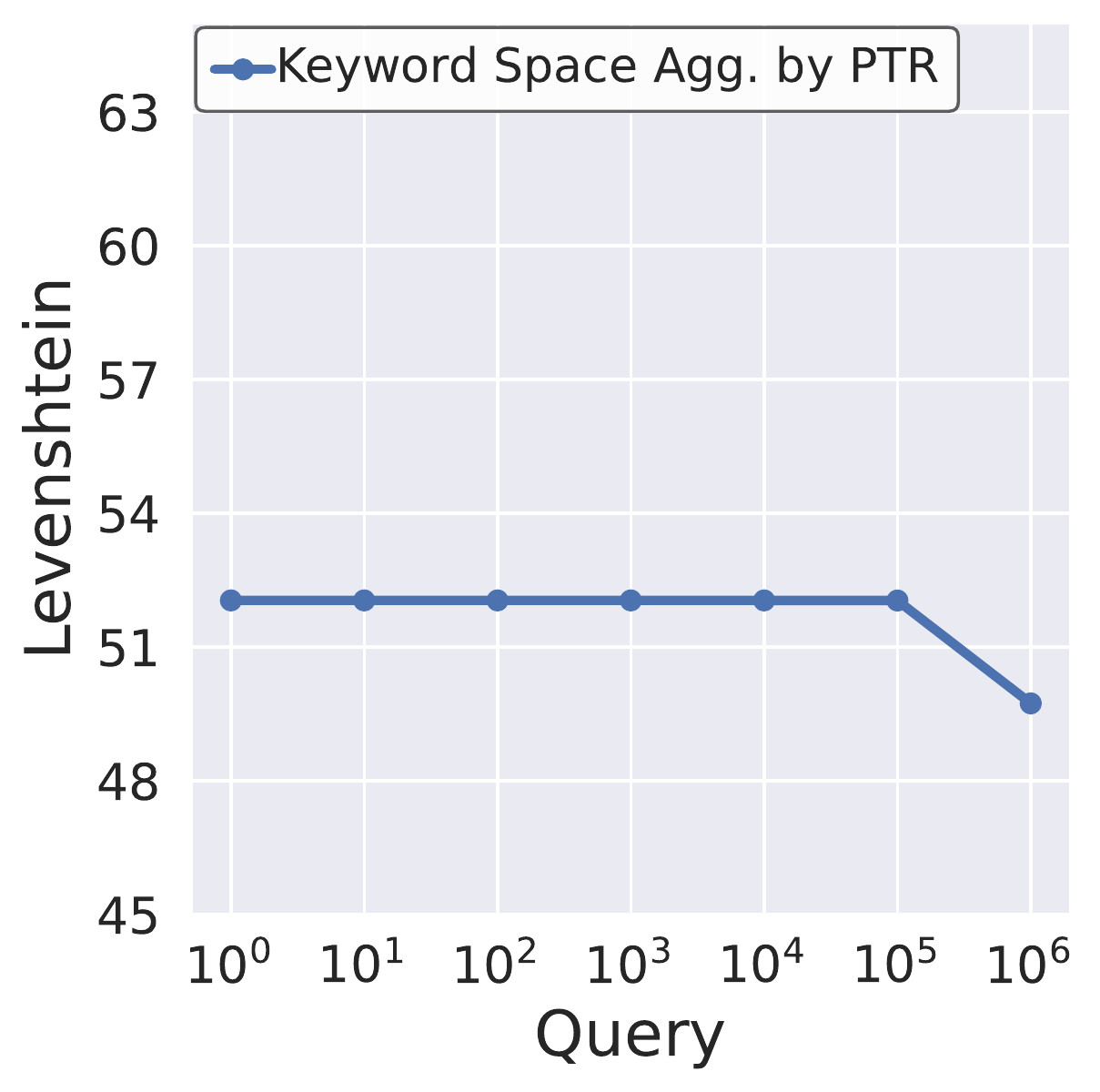}  \\
  \end{tabular}
  \vspace{-3mm}
      \caption{Ablation studies on varying numbers of queries. }
      \label{Fig:abl_query_dqa}
  \end{figure}

\subsection{Dialog summarization task}
\label{appsub:aba_lang_ds}

\textbf{Effectiveness across numbers of public candidates (\cref{Fig:ab_dqa_can}).} We study the performance of the embedding space aggregation method with varying numbers of public candidates. As anticipated, an increased number of public candidates leads to improved performance. The ideal scenario would involve maximizing the number of public candidates. However, a larger pool of candidates also needs more API queries, thereby increasing the overall cost.

\begin{figure}[H]
  \centering
    \includegraphics[width=0.50\textwidth]{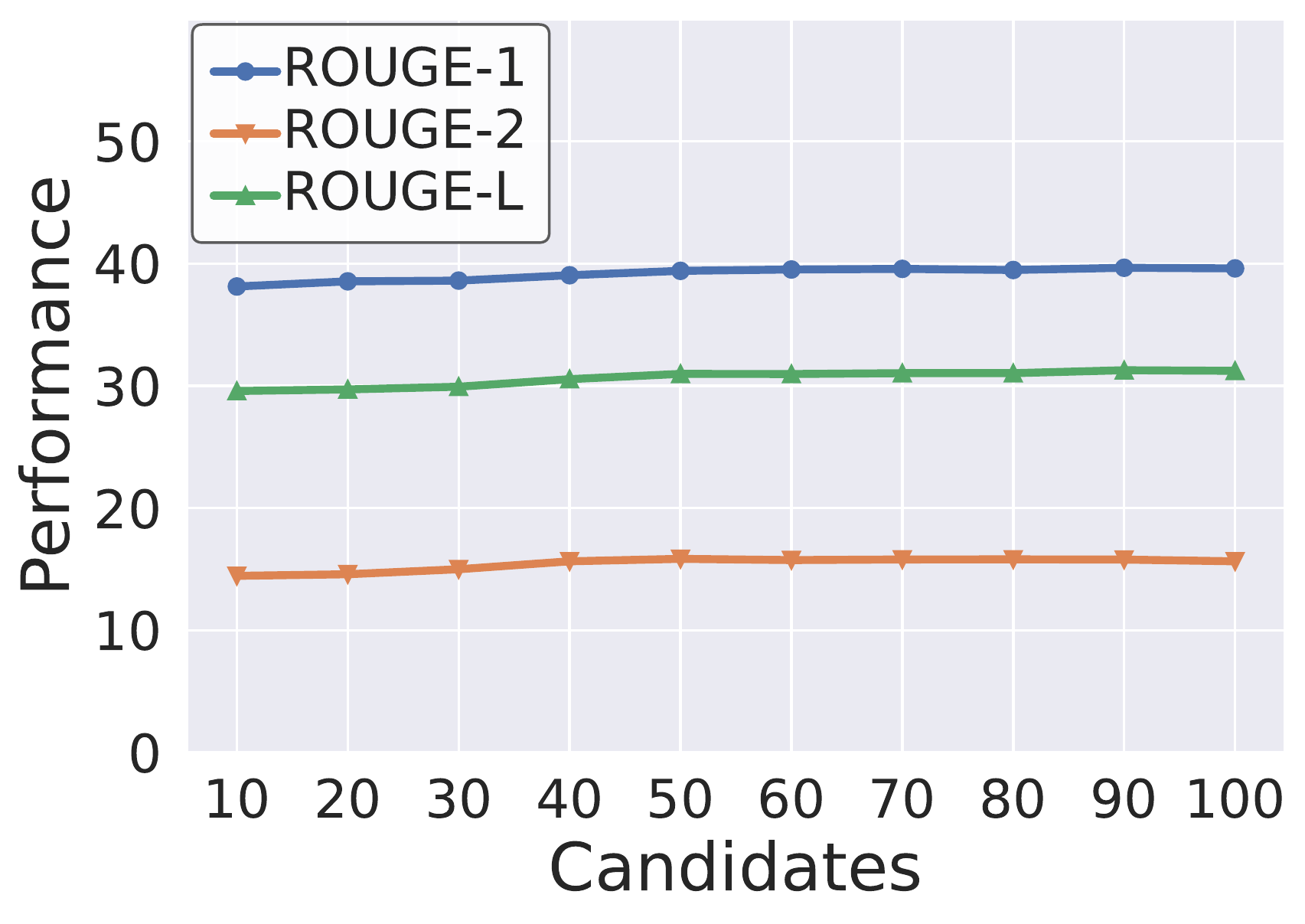} 
    \caption{Performance across numbers of the public candidates (i.e., number of answers generated by zero-shot predictions) for embedding space aggregation method.}
    \label{Fig:ab_dqa_can}
  \end{figure}

\textbf{Effectiveness across numbers of ensembles (\cref{Fig:abl_ensemble_full}).} Here, we present the full results of ablation studies of the number of ensembles.

  \begin{figure}[H]
  \centering
  \begin{tabular}{cccc}
    \includegraphics[width=0.33\textwidth]{figure/embedding_ab_ensembler1.pdf} & 
     \includegraphics[width=0.33\textwidth]{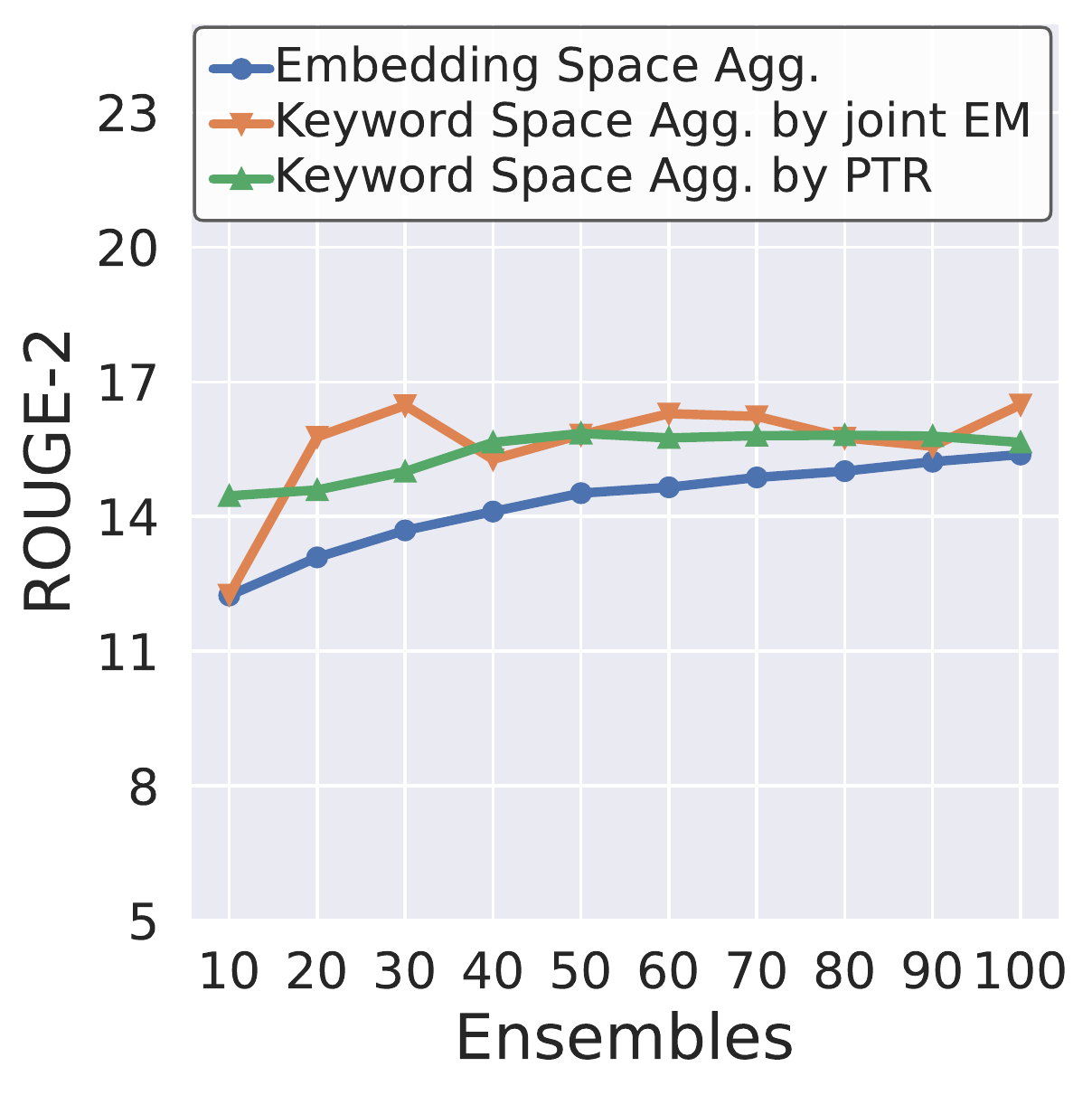} 
   \includegraphics[width=0.33\textwidth]{figure/embedding_ab_ensemblerl.pdf}  \\
  \end{tabular}
  \vspace{-3mm}
      \caption{Ablation studies on varying numbers of ensembles. Full results of \cref{Fig:abl_query_ensemble}(a)}
      \label{Fig:abl_ensemble_full}
  \end{figure}

\textbf{Effectiveness across numbers of queries (\cref{Fig:abl_query_full}).} Here, we present the full results of ablation studies of the number of queries. 

  \begin{figure}[H]
  \centering
  \begin{tabular}{cccc}
    \includegraphics[width=0.33\textwidth]{figure/embedding_ab_query_r1.pdf} & 
     \includegraphics[width=0.33\textwidth]{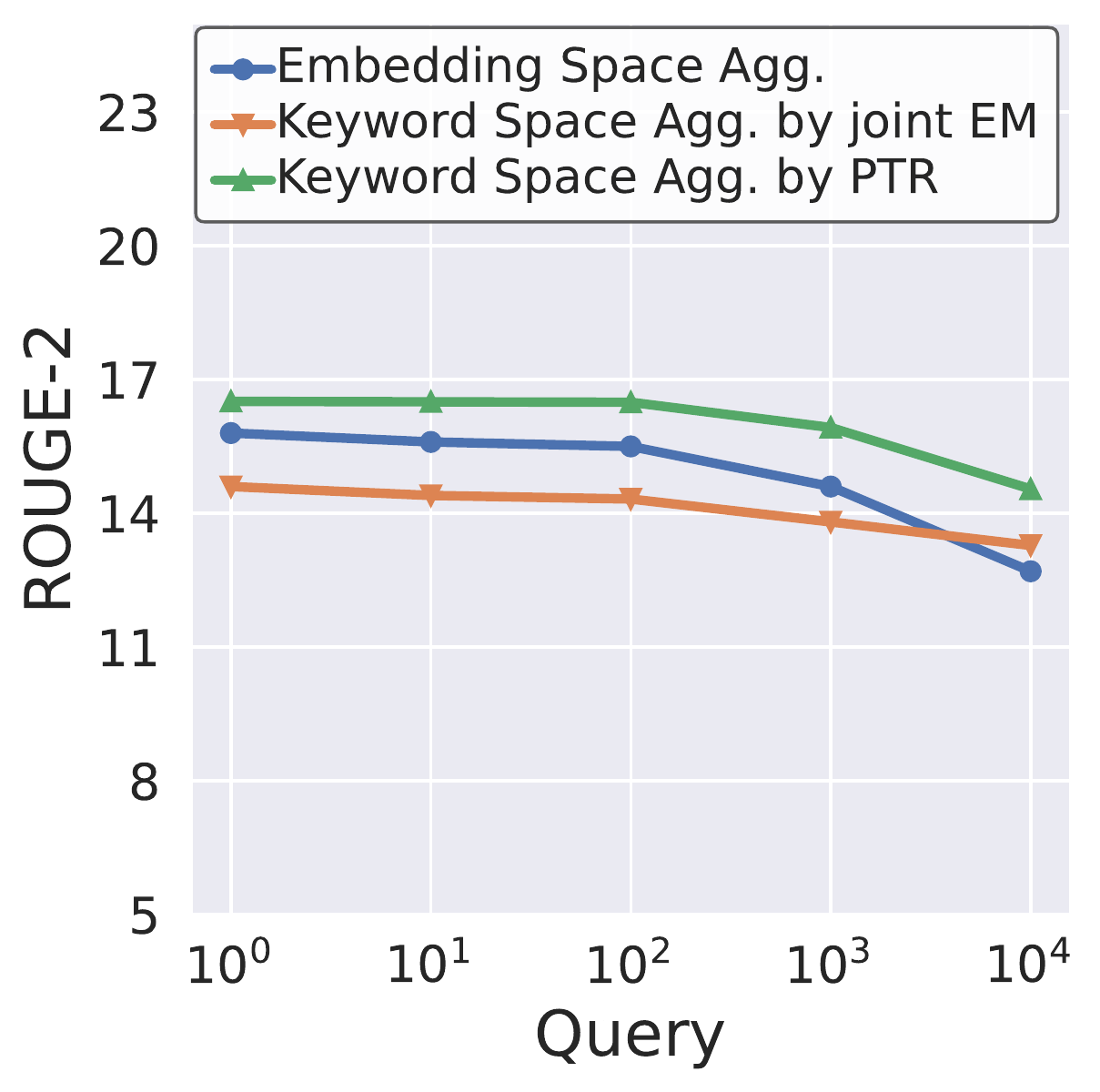} 
   \includegraphics[width=0.33\textwidth]{figure/embedding_ab_query_rl.pdf}  \\
  \end{tabular}
  \vspace{-3mm}
      \caption{Ablation studies on varying numbers of queries.  Full results of \cref{Fig:abl_query_ensemble}(b)}
      \label{Fig:abl_query_full}
  \end{figure}

\newpage

\end{document}